\documentclass{tlp}
\usepackage{aopmath}
\usepackage{times}
\usepackage{helvet}
\usepackage{epstopdf}  
\usepackage{courier}
\usepackage{color}
\usepackage{xspace}
\usepackage{verbatim}
\usepackage{hyperref}
\usepackage{natbib}
\usepackage{todonotes}

\def \cyp#1{(\citeyear{#1})}

\newtheorem{thm}{Theorem}
\newtheorem{example}{Example}

\newcommand\specialref{}

\newcommand{\cC}{\mathcal{C}}
\newcommand{\cE}{\mathcal{E}}

\newcommand{\cA}{\mathcal{A}}
\newcommand{\cP}{\mathcal{P}}

\newcommand{\cV}{\mathcal{V}}

\newcommand{\At}{\mathit{At}}

\newcommand{\tcb}{\textcolor{black}}
\newcommand{\tcg}{\textcolor{black}}
\newcommand{\tcm}{\textcolor{black}}

\newcommand{\eqref}[1]{(\ref{#1})}
\def\cite{\citep}
\makeatletter
\newcommand{\labitem}[2]{%
\def\@itemlabel{\textbf{#1}}
\item
\def\@currentlabel{#1}\label{#2}}
\makeatother

\emergencystretch=\hsize
\usepackage[english,american]{babel}
\usepackage{graphicx}

\usepackage{mdwlist}    

\newcommand{\bi}{\begin{itemize}}
\newcommand{\ei}{\end{itemize}}

\def\at{\mathcal{B}}
\def\dec{\Delta}
\def\lrar{\Rightarrow}

\newcommand{\tbeg}{\langle}
\newcommand{\tend}{\rangle}

\newcommand{\lpnot}{\mbox{not}\;\,}

\newcommand{\hif}{\leftarrow}

\newcommand{\st}{
}

\newcommand{\ol}{\overline}

\newcommand{\fail}{\hbox{{\em Failstate}}}
\newcommand{\ezg}{\hbox{{\sc ez}}}

\newcommand{\ezsmg}{\hbox{{\sc ezsm}}}

\newcommand{\rd}{\hbox{{\em Decide}}\xspace}
\newcommand{\rup}{\hbox{{\em Unit Propagate}}\xspace}
\newcommand{\runf}{\hbox{{\em Unfounded}}\xspace}
\newcommand{\rb}{\hbox{{\em Backtrack}}\xspace}
\newcommand{\rfail}{\hbox{{\em Fail}}\xspace}
\newcommand{\rap}{\hbox{{\em ASP-Propagate}}\xspace}
\newcommand{\rcp}{\hbox{{\em CP-Propagate}}\xspace}

\newcommand{\rlp}{\hbox{{\em Learn}}\xspace}
\newcommand{\rlt}{\hbox{{\em Learn$^t$}}\xspace}
\newcommand{\rpr}{\hbox{{\em Restart}}\xspace}
\newcommand{\rcr}{\hbox{{\em Restart$^t$}}\xspace}

\newcommand{\aspindent}{\hspace*{.5in}}
\newcommand{\aspindhif}{\hspace*{.18in}}

\newcommand{\suggestion}[2]{}

\def\ar{\leftarrow} 
\def\beq{\begin{equation}} 
\def\eeq#1{\label{#1}\end{equation}} 
\def\beq{\begin{equation}}
\def\eeq#1{\label{#1}\end{equation}}
\def\ba{\begin{array}}
\def\ea{\end{array}}
\def\oi#1{\begin{itemize}\item {#1}\end{itemize}} 

\def\bsmall{}
\def\esmall{}

\def\wseq{{\sc wseq}}
\def\is{{\sc is}}
\def\rf{{\sc rf}}

\def\seq{{\sc seq}++}

\def\aspcomp{{\sc aspcomp}\xspace}
\def\gringo{{\sc gringo}\xspace}
\def\clasp{{\sc clasp}\xspace}

\def\aspmt{{\sc aspmt2smt}\xspace}
\def\ezsmt{{\sc ezsmt}\xspace}

\def\bprolog{{\sc bprolog}\xspace}

\def\smodels{{\sc smodels}\xspace}
\def\cmodels{{\sc cmodels}\xspace}
\def\ezcsp{\textsc{ezcsp}\xspace}
\def\ez{\textsc{ez}\xspace}
\def\clingcon{{\sc clingcon}\xspace}
\def\gecode{{\sc gecode}\xspace}
\def\acsolver{{\sc acsolver}\xspace}
\def\idp{{\sc idp}}
\def\ac{{\sc ac}}
 \def\minisat{{\sc minisat}\xspace}

\def\bbox{\textit{black-box}\xspace}
\def\cbox{\textit{clear-box}\xspace}
\def\gbox{\textit{grey-box}\xspace}

\makeatletter
\def\@copyrightspace{\relax}
\makeatother
\begin{document}

\selectlanguage{american}

\title{Constraint Answer Set Solver EZCSP and
Why Integration Schemas Matter\footnote{{\bf This version of the paper corrects   inaccurate claims occurring in Section 2.3 and the beginning of Section~3 of the paper that appeared in print at  TPLP 17(4): 462-515 (2017).  We are grateful to Sara Biavaschi and Agostino Dovier for bringing this issue to our attention. The changes are marked by footnotes.}}}


\author[M. Balduccini and Y. Lierler]
{MARCELLO BALDUCCINI\\
Department of Decision \& System Sciences\ \\
 Saint Joseph's University\\
              \email{marcello.balduccini@gmail.com} \\
\and YULIYA LIERLER \\
Computer Science Department \\
 University of Nebraska at Omaha \\
 \email{ylierler@unomaha.edu}
}

\date{}

\maketitle

\begin{abstract}
Researchers in 
answer set programming and constraint programming have spent
significant efforts in the development of
hybrid languages and solving algorithms combining the strengths of these
traditionally separate fields. These efforts resulted 
in a new research area: constraint answer set programming. 
Constraint answer set programming languages and  systems proved to be  successful at
providing  declarative, yet efficient solutions to problems
involving hybrid reasoning tasks. 
One of the main contributions of this paper is  the first
comprehensive account of  the constraint answer set language and solver {\ezcsp},
 a mainstream representative of this 
research area that has been used in various successful applications.
We also develop an extension of the transition systems proposed by Nieuwenhuis et al. in 2006 to capture Boolean satisfiability solvers. We use this extension to describe 
the   \ezcsp algorithm and prove formal claims about it.
The design and algorithmic details behind \ezcsp clearly demonstrate that the development of the hybrid systems of this kind is challenging. 
Many questions arise when one faces various design choices in an attempt to maximize system's benefits.
One of the key decisions that a developer of a hybrid solver makes 
is settling on a particular integration schema within its implementation. Thus, another important contribution of this paper is a 
thorough case study  based on \ezcsp, focused on the various integration schemas that it provides.\\
\emph{Under consideration in Theory and Practice of Logic Programming (TPLP).}

\end{abstract}

\section{Introduction}
Knowledge representation and automated reasoning are  areas of Artificial Intelligence 
 dedicated to understanding and  automating 
 various aspects of reasoning. Such traditionally separate fields of AI
  as answer set programming~(ASP)~\cite{nie99,mar99,bre11},
  propositional satisfiability~(SAT)~\cite{gom08}, 
 constraint (logic) programming (CSP/CLP)~\cite{ros08,jm94} are all
 representatives of distinct directions of research in automated reasoning.  
The algorithmic techniques
 developed in subfields of automated reasoning are often suitable for
 distinct reasoning tasks. For example, ASP proved to be
 an effective  tool for formalizing elaborate planning tasks, whereas
 CSP/CLP is efficient in solving difficult scheduling
 problems. However, when solving complex practical problems, such as scheduling
 problems  involving elements of planning or defeasible statements,
 methods that go beyond traditional ASP and CSP are sometimes
 desirable. By allowing one to leverage specialized algorithms for solving different parts of the problem at hand, these methods may yield better performance than the traditional ones. Additionally, by allowing the use of constructs that more closely fit each sub-problem, they may yield solutions that conform better to the knowledge representation principles of flexibility, modularity, and elaboration tolerance.
This has led, in recent years, to the development of a plethora of {\em hybrid} approaches that combine
algorithms and systems from different AI subfields.  
Constraint logic programming~\cite{jm94},  satisfiability modulo theories~(SMT)~\cite{nie06},
HEX-programs~\cite{eit05a}, and VI-programs~\cite{cal07}
are all examples of this current.
Various projects have focused on the intersection of ASP and CSP/CLP,
which  resulted in the development of a new field of
study, often called {\em constraint answer set programming}~(CASP)
\cite{elk04,mel08,geb09,bal09,dre11a,lier14}.

Constraint answer set programming allows one to combine the best of two
  different automated reasoning worlds: 
(1) the non-monotonic modeling capabilities and SAT-like solving
technology of ASP and (2) constraint processing techniques for 
effective reasoning over non-Boolean constructs.
This new area has 
already demonstrated promising results, including
the development of CASP solvers {\acsolver}~\cite{mel08},
{\clingcon}~\cite{geb09}, {\ezcsp}~\cite{bal09}, 
{\sc idp}~\cite{idp}, {\sc inca}~\cite{dre11a}, {\sc
  dingo}~\cite{jan11},  {\sc mingo}~\cite{liu12}, \aspmt~\cite{Bartholomew2014}, and~\ezsmt~\cite{sus16a}.  
CASP opens new horizons for
 declarative programming applications. 
 For instance, research by Balduccini~\cyp{bal11} on the design of
 CASP language \textsc{ezcsp} and on the corresponding solver, which
 is nowadays one of the mainstream representatives of CASP systems,
\tcm{ yielded an elegant, declarative solution to a complex industrial scheduling
 problem.}

\tcm{Unfortunately, achieving the level of integration  of CASP languages and systems requires  
nontrivial expertise in multiple areas, such as SAT, ASP and
CSP.}
%
%
\tcg{The crucial message transpiring from the developments in the CASP research area  is the  need for  standardized techniques to integrate computational
methods spanning these multiple research areas. 
We argue for
undertaking an effort to mitigate the difficulties  of designing hybrid
reasoning systems by identifying general principles for their
development and studying the implications of various design
choices. Our work constitutes a step in this
direction. Specifically, the main contributions of our work are: }
\begin{enumerate}
 \item \tcg{The paper provides the  first comprehensive account of the constraint answer set
   solver {\ezcsp} \cite{bal09},  a long-time representative of the CASP subfield.
We define the language of \ezcsp and illustrate its use on several examples.
We also account for algorithmic and implementation details behind  \ezcsp.}
\item \tcg{To  present the   \ezcsp algorithm and prove formal claims about the system, 
we develop an extension of the transition systems
proposed by Nieuwenhuis et al.~(\citeyear{nie06}) for capturing
SAT/SMT algorithms. This extension is well-suited for formalizing the behavior of the \ezcsp solver.}
\item \tcg{We also  conduct  a case study exploring a crucial aspect in building hybrid 
systems -- the integration schemas of participating solving
methods. This allows us to
shed  light on the costs and benefits of this key design choice in hybrid systems.
  For the case study, \tcm{we use
{\ezcsp} as a research tool and study its performance with three  integration schemas: ``{\bbox}'',
``\gbox'', and ``\cbox''.} One of the main conclusions of the study is that there is no single
choice of integration schema that achieves best performance in all
cases. As such, the choice of integration schema should be made as easily
configurable as it is the choice of particular branching heuristics in
SAT or ASP solvers. 
\tcm{The} work on analytical and architectural aspects \tcm{described in this paper} shows how this can be achieved.} 
\end{enumerate}

\st
We begin this paper with a review of the ASP and CASP formalisms. In
Section~\ref{sec:ezcsplang} we present the {\ezcsp} language. 
In Section~\ref{sec:integration} we provide a broader context to our
study by drawing a 
parallel between CASP and SMT solving. 
Then we review the
integration schemas used in the design of hybrid solvers focusing on the schemas
 implemented in {\ezcsp}.  
Section~\ref{sec:ezcspsolver} provides a comprehensive account of
algorithmic aspects of {\ezcsp}. 
 Section~\ref{sec:domain} introduces the details of the ``integration schema'' case study.
In particular, it provides details on the application domains considered, namely, Weighted Sequence, Incremental Scheduling, and
Reverse Folding. The section also discusses the variants
of the encodings we compared.
Experimental results and their analysis form Section~\ref{sec:experiments}. Section~\ref{sec:relsys} provides a brief overview of CASP solvers. The conclusions are stated in Section~\ref{sec:concl}.

Parts of this paper have been earlier
 presented at ASPOCP 2009~\cite{bal09}
and at PADL~2012~\cite{bl12}.

\section{Preliminaries}

\subsection{Regular Programs}\label{sec:prelim}
A {\sl regular (logic) program} is a finite set of rules of the form
\beq
\ba {l}
a_0\ar a_1,\dots, a_l,not\ a_{l+1},\dots,not\ a_m,
not\ not\ a_{m+1},\dots,not\ not\ a_n,
\ea
\eeq{e:rule}
where $a_0$ is $\bot$ (false) or an atom, and
each $a_i$ ($1\leq i\leq n$) is an atom so that $a_i\neq a_j$ ($1\leq i<j\leq l$),
$a_i\neq a_j$ ($l+1\leq i<j\leq m$), and $a_i\neq a_j$ ($m+1\leq i<j\leq n$).
This  is a special case of 
programs with nested expressions \cite{lif99d}.
The expression $a_0$ is the \emph{head} of a
rule~(\ref{e:rule}). 
If $a_0=\bot$, we often omit $\bot$ from the notation. We call such
rules {\em denials}. 
We call 
the right hand side of the arrow in (\ref{e:rule}) 
 the {\sl body}. If a body of a rule is empty, we call such rule a {\em fact} and omit the $\ar$ symbol. We also ignore the order of the elements in the rule. For example, rule $a \ar b,c$ is considered identical to $a \ar c,b$. 
If $B$ denotes
the {\sl body} of~(\ref{e:rule}),
 we write $B^{pos}$ for the elements
occurring in the {\sl positive} part
of the body, i.e., $B^{pos}=\{a_1,\dots, a_l\}$.
We frequently identify the body of~(\ref{e:rule}) 
with the conjunction of its elements (in which $not\ not$ is dropped and 
$not$ is replaced with the classical negation connective $\neg$):
\beq
  a_1\wedge \dots\wedge  a_l\wedge  \neg a_{l+1}\wedge \dots\land\neg
  a_m \wedge   a_{m+1}\wedge \dots\land a_n.
\eeq{e:bodycl}
Similarly, we often interpret a rule~(\ref{e:rule}) as a clause
\beq
a_0\vee \neg  a_1\vee \dots\vee \neg a_l\vee   
a_{l+1}\vee \dots \vee a_m \vee \neg a_{m+1}\vee \dots \vee \neg a_n
\eeq{e:cl}
In the case when  $a_0=\bot$ in~(\ref{e:rule}),  $a_0$ is absent in (\ref{e:cl}).
Given a program $\Pi$, we write $\Pi^{cl}$ for the set of clauses of the form~(\ref{e:cl}) 
corresponding to the rules in $\Pi$.
 
\paragraph{Answer sets}
An {\em alphabet} is a set of atoms.
The semantics of logic programs relies on the notion of answer sets, which are sets of
atoms. \tcb{A {\em literal} is an atom $a$ or its negation $\neg a$.
We say that a set $M$ of literals is {\em complete} over alphabet
$\sigma$ if, for any atom $a$ in $\sigma$, either $a\in M$ or $\neg
a\in M$. } 
It is easy to see how a set $X$ of atoms over some
alphabet~$\sigma$ can be identified with a complete and consistent
set of literals over~$\sigma$
(an interpretation): \[\{a\mid a\in X\}\cup\{\neg a\mid a\in\sigma\setminus X\}.\] 
We now restate the definition of an answer set due to Lifschitz et
al.~(\citeyear{lif99d})   in a form convenient for
our purposes. 
By $\At(\Pi)$ we denote the set of all atoms that occur  in $\Pi$.
The \emph{reduct} $\Pi^X$ of a regular program~$\Pi$ with
respect to set $X$ of atoms over $\At(\Pi)$ is obtained from~$\Pi$ by deleting each
rule~\eqref{e:rule} such that $X$ does not satisfy its body (recall
that we identify its body with  \eqref{e:bodycl}), and
replacing each remaining rule~\eqref{e:rule} by $a_0\ar B^{pos}$.
A set $X$ of atoms is an
\emph{answer set} of a regular program $\Pi$ if it is subset minimal among
the sets of atoms satisfying $(\Pi^X)^{cl}$. 
For example, consider a program consisting of a single rule
$$
a\ar\ not\ not\ a.
$$
This program has two answer sets: set $\emptyset$ and set $\{a\}$.
Indeed, $(\Pi^\emptyset)^{cl}$ is an empty set of clauses so that $\emptyset$ is subset minimal among the sets of atoms that satisfies  $(\Pi^\emptyset)^{cl}$.
On the other hand, $(\Pi^{\{a\}})^{cl}$ consists of a single clause $a$.
Set $\{a\}$ is subset minimal among the sets of atoms that satisfies $(\Pi^{\{a\}})^{cl}$.

A \emph{choice rule}
construct $\{a\}\ar B$~\cite{nie00} of 
the {\sc lparse} language
 can be seen as an abbreviation for a rule
\hbox{$a\ar\ not\ not\ a, B$}~\cite{fer05b}. We adopt this abbreviation in
the rest of the paper. 

\begin{example}
Consider  the regular program 
\begin{equation}\label{ex:acp}
\ba l
  \{switch\}.\\
  lightOn\ar\ switch, not\ am.\\
  \ar not\ lightOn. \\
  \{am\}.\\
\ea
\end{equation}
Intuitively, the rules of the program  state the following:
\begin{itemize}
\item action {\em switch} is exogenous,
\item {\em light} is {\em on}
 only if an action  {\em switch}
occurs during the non-{\em am} hours,
\item it is impossible that {\em light} is not {\em on} (in other words, {\em light} must be {\em on}).
\item it is either the case that these are {\em am} hours or not,
\end{itemize}
This program's only answer set is $\{switch, \ lightOn\}$.
\end{example}

We now state an important result that summarizes the effect of adding denials to a program. 
For a set $M$ of literals, by~$M^{+}$   we denote  the set of
positive  literals in $M$. For instance, $\{a,c,\neg b\}^{+}=\{a,c\}$.
\begin{thm}[Proposition~2 from~\cite{lif99d}]\label{prop:constraints} 
 For  a program $\Pi$, a set $\Gamma$ of denials, and 
a consistent and complete set $M$ of literals over $\At(\Pi)$, 
$M^{+}$ is an answer set of $\Pi\cup\Gamma$ if and only if 
$M^{+}$ is an answer set of $\Pi$ and $M$ is a model of $\Gamma^{cl}$. 
\end{thm}

\paragraph{Unfounded sets}
For a literal $l$, by $\overline l$ we denote its complement.
For a conjunction (disjunction) $B$ of literals,  $\overline B$ stands for
 a disjunction (conjunction) of the complements of literals.
 For instance, \hbox{$\ol{a\wedge \neg b}=\neg a \vee b$.}
 We sometimes associate disjunctions and conjunctions of literals with the sets containing these literals. 
 For example, conjunction $\neg a \wedge b$ and disjunction $\neg a \vee b$ are associated with the set
  $\{\neg a,~b\}$ of literals.   
 By $Bodies(\Pi,a)$ 
we denote  the set of the bodies of all rules of program~$\Pi$ with the head~$a$
(including the empty body that can be seen as $\top$).  

A set $U$ 
of atoms occurring in a program~$\Pi$ is \emph{unfounded}~\cite{van91,lee05} 
on a consistent set $M$ of literals
with respect to $\Pi$ if for every $a\in U$ and every 
$B\in Bodies(\Pi,a)$, $M\cap \overline{B}\not=\emptyset$ or 
$U\cap B^{pos}\neq\emptyset$. 
We say that a consistent and complete set $M$ of literals over $\At(\Pi)$
is a {\em model} of $\Pi$ if it is a model  of ${\Pi}^{cl}$.

We 
now state a result that can  be seen as an alternative
way to characterize  answer sets of a program.

\begin{thm}[Theorem on Unfounded Sets from~\cite{lee05}]\label{prop:leone} 
 For  a program $\Pi$ and a consistent and complete set $M$ of
 literals over $\At(\Pi)$, 
$M^{+}$ is an answer set of~$\Pi$ if and only if 
$M$ is a model of~$\Pi$ and
$M$ contains no non-empty  subsets unfounded on~$M$ 
with respect to~$\Pi$.
\end{thm}
Theorem~\ref{prop:leone} is essential in understanding key
features of modern answer set solvers. It provides a description of
properties of answer sets that are utilized by so called 
``propagators'' of  solvers. 
Section~\ref{sec:ezcspsolver} relies on these properties.

\subsection{Logic Programs with Constraint Atoms}\label{sec:LPCA}

A constraint satisfaction problem (CSP) is defined as a triple
$\langle X,D,C \rangle$, where $X$ is a set of variables, $D$ is a
domain -- a (possibly infinite) set of values -- 
and $C$ is a set of constraints. Every constraint is  a
pair $\langle t,R \rangle$, where $t$ is
an $n$-tuple of variables and $R$ is an $n$-ary relation on $D$. 
\tcb{When arithmetic constraints are considered, it is common to
  replace explicit representations of relations as collections of
  tuples by arithmetic expressions. For instance, for a domain of
  three values $\{1,2,3\}$ and binary-relation $R$ consisting of
  ordered pairs $(1,1),(2,2)$, and $(3,3)$, we can abbreviate the
  constraint $\langle x,y,R \rangle$ by the expression  $x=y$. We
  follow this convention in the rest of the paper. 
}

An
{\em evaluation}
of the variables is a function from the set of variables to the domain
of values, $\nu:X \rightarrow D$. An evaluation~$\nu$ {\em satisfies} a constraint
$\langle (x_1,\ldots,x_n),R \rangle$ if $(v(x_1),\ldots,v(x_n)) \in R$. A
{\em solution} is an evaluation that satisfies all constraints.

For a constraint $c=\langle t,R \rangle$, where $D$ is the domain of its
variables and $k$ is the arity of~$t$, we call 
the constraint $\overline{c}=\langle t,D^k\setminus R\rangle$  the \emph{complement}
 of $c$. Obviously, an evaluation of variables in $t$
satisfies $c$ if and only if it does not satisfy $\overline{c}$.

For a set $M$ of literals and alphabet~$\at$, by $M_{|\at}$  we denote   the set of
 literals over alphabet~$\at$ in $M$. For example, $\{\neg a,~b,c\}_{|\{a,b\}}=\{\neg a,~b\}$.


A {\em logic program with constraint atoms} (CA program) is a quadruple
\[\langle \Pi,\cC,\gamma,D \rangle,\] 
where
\begin{itemize}
\item $\cC$ is an alphabet, 
\item $\Pi$ is a regular logic program such that 
(i) $a_0\not\in\cC$ for every rule (\ref{e:rule}) in $\Pi$ 
and (ii) $\cC\subseteq \At(\Pi)$,
\item $\gamma$ is  a function from $\cC$ to constraints, and
\item $D$ is a domain.
\end{itemize}
We refer to the elements of alphabet $\cC$ as \emph{constraint} atoms.
We call all atoms occurring in $\Pi$ but not in $\cC$ \emph{regular}.
To distinguish constraint atoms from the constraints 
to which these atoms  are mapped,  
we use bars to denote that an expression  is a
constraint atom. For instance, $|x<12|$ and  $|x\geq 12|$ denote
constraint atoms. Consider alphabet $\cC_1$ that consists of these two
constraint atoms and a function  $\gamma_1$ that 
maps atoms in  $\cC_1$ to constraints as follows:
  $\gamma_1(|x<12|)$ maps
to an inequality $x<12$, whereas 
$\gamma_1(|x\geq 12|)$ maps
to an inequality $x\geq 12$.
Clearly,~$\overline{\gamma_1(|x<12|)}$ maps into 
an inequality $x\geq 12$; similarly
$\overline{\gamma_1(|x\geq 12|)}$ maps into 
an inequality $x< 12$. 

\begin{example}\label{example:p1}
Here we present a sample CA program  
\beq
\cP_1=\langle \Pi_1,\cC_1,\gamma_1,D_1 \rangle,
\eeq{eq:sample}
where 
$D_1$ is a range of integers from $0$ to $23$ and
 $\Pi_1$ is  a regular program 
 \begin{equation}\label{ex:acp2}
\ba l
  \{switch\}.\\
  lightOn\ar\ switch, not\ am.\\
  \ar not\ lightOn.\\ 
  \{am\}.\\
  \ar not\ am,\ |x< 12|. \\
  \ar am,\ |x\geq 12|. \\
\ea
\end{equation}
The first four rules of $\Pi_1$ follow the lines of~\eqref{ex:acp}.
The last two rules intuitively state that
\begin{itemize}
\item it is impossible that these are not {\em am} hours while
  variable $x$ has a value less than $12$,  
\item it is impossible that these are {\em am} hours while variable
  $x$ has a value greater or equal to $12$. 
\end{itemize}
Note how $x$  represents specific hours of a day. Also worth noting is the fact that $x$ has a global scope. This is different from the traditional treatment of variables in CLP, Prolog, and ASP.
\end{example}

Let 
$\cP=\langle \Pi,\cC, \gamma,D \rangle$ be a CA  program. 
By $\cV_{\cP}$ we denote the set of variables occurring in the constraints $\{\gamma(c)\mid c\in
\cC\}$.
For instance, $\cV_{\cP_1}=\{x\}$.
By $\Pi[\cC]$ we denote~$\Pi$ extended
with choice rules $\{c\}$ for each constraint atom $c\in\cC$. We call program  $\Pi[\cC]$ an
{\em asp-abstraction} of $\cP$. For example,
an asp-abstraction $\Pi_1[\cC_1]$ of 
any CA program whose first two elements of its quadruple are $\Pi_1$ and $\cC_1$
consists of rules~\eqref{ex:acp2} and the following choice rules
\[
\ba l
  \{|x<12|\}\\
  \{|x\geq 12|\}.\\
\ea
\]
Let $M$ be a consistent set of literals over $\At(\Pi)$. 
By $K_{\cP,M}$ we denote the following constraint satisfaction problem 
$$\langle \cV,~D,~ \{\gamma(c)|c\in M_{|\cC}, c\in\mathcal{C} \}\cup  \{\overline{\gamma(c)}|\neg c\in M_{|\cC}, c\in\mathcal{C} \}\rangle,$$ 
where $\cV$ is the set of variables occurring in the constraints of
the last element of the triple above. 
We call this constraint satisfaction problem a {\em csp-abstraction} of $\cP$ with respect to  $M$. 
For instance, a csp-abstraction of~$\cP_1$ w.r.t. $\{|x\geq 12|,~\neg|x< 12|,~ lightOn\}$, or
$K_{\cP_1,\{|x\geq 12|,~\neg|x< 12|,~ lightOn\}}$, is
\beq\langle \{x\},D_1, \{x\geq 12\}\rangle.
\eeq{ex:csp}
It is easy to see that  $\cV_{\cP}$ consists of the variables that
occur in a csp-abstractions of $\cP$ w.r.t. any
consistent sets of literals over $\At(\Pi)$.

Let $\cP=\langle \Pi,\cC,\gamma,D \rangle$ 
be a CA program and  $M$ 
be  a consistent and complete set  of literals over
$\At(\Pi)$.
We say that $M$
  is an {\em answer set} of 
$\cP$ if  
\begin{enumerate}
\labitem{(a1)}{as.1}  $M^{+}$ is an answer set of $\Pi[\cC]$ and
\labitem{(a2)}{as.2}  the constraint satisfaction problem $K_{\cP,M}$
 has a solution.
\end{enumerate}
\tcb{Let~$\alpha$ be an evaluation from the set $\cV_{\cP}$ of variables  to the set $D$ of values.}
We say that a pair $\langle M, \alpha \rangle$ is an {\em extended
   answer set} of $\cP$ if $M$ is an answer set of $\cP$ and $\alpha$
 is a solution to~$K_{\cP,M}$.

\begin{example}\label{ex:answerset}
Consider sample CA program $\cP_1=\langle \Pi_1,\cC_1,\gamma_1,D_1 \rangle$ given in
\eqref{eq:sample}.
Consistent and complete set 
\[
M_1=\{switch, lightOn, \neg am, \neg |x<12|, |x\geq 12|\}
\]
of literals over $\At(\Pi_1)$ is such that $M_1^{+}$ is the answer set of $\Pi_1[\cC_1]$.
The constraint satisfaction problem $K_{\cP_1,M_1}$ is presented in~\eqref{ex:csp}.
Pairs
$$ \langle M_1, x=12 \rangle 
$$
and
$$ \langle M_1, x=23 \rangle 
$$
are two among twelve extended answer sets of
program~\eqref{eq:sample}.
\end{example}

\subsection{CA Programs and Weak Answer Sets}
In the previous section we introduced CA programs that capture programs that a CASP solver such as {\clingcon} processes. The \ezcsp solver interprets similar programs slightly differently.  To illustrate the difference we introduce the notion of a weak answer set for a CA program and discuss the differences with earlier definition.

Let $\cP=\langle \Pi,\cC,\gamma,D \rangle$ 
be a CA program and  $X$ 
be  a set of atoms  over
$\At(\Pi)$.
We say that $X$
is a {\em weak answer set} of 
$\cP$ if  
\begin{enumerate}
        \labitem{(w1)}{as1.1}  $X$ is an answer set of $\Pi[\cC]$ and
        \labitem{(w2)}{as2.2}  the constraint satisfaction problem 
        \beq 
        \langle \cV_{\cP},~D,~ \{\gamma(c)|c\in X_{|\cC} \}\rangle,
        \eeq{eq:kpm}    
 has a solution.
\end{enumerate}
\tcb{Let~$\alpha$ be an evaluation from the set $\cV_{\cP}$ of variables  to the set $D$ of values.}
We say that a pair $\langle X, \alpha \rangle$ is an {\em extended
        weak answer set} of $\cP$ if $X$ is an answer set of $\cP$ and $\alpha$
is a solution to~\eqref{eq:kpm}.

The key difference between the definition of an answer set and a weak answer set of a CA program lies in their conditions~\ref{as.2} and~\ref{as2.2}. (It is obvious that we can always identify a complete and consistent set of literals with the set of its atoms.)
 To illustrate the difference between the two semantics, consider simple CA program:
$$
\ba{l}
night \ar |x< 6|. \\
am \ar|x<12|. \\
\ea
$$
This program has  three answer sets and four weak answer sets that we present in the following table. 
\[\ba{l|l}
\hbox{Answer Sets:}& \hbox{Weak Answer Sets:}\\
\{night, am, |x<6|,  |x<12|\}&\{night, am, |x<6|,  |x<12|\}\\
\{\neg night, am, \neg |x<6|,  |x<12|\}&\{am,  |x<12|\}\\
\{\neg night, \neg am, \neg |x<6|,  \neg |x<12|\}&\emptyset\\
& \{night, |x<6|\}
\ea
\]
Note how the last weak answer set listed yields an unexpected solution, as it suggests that it is currently night but not am hours.

Another sample program is due to Sara Biavaschi and Agostino Dovier\footnote{This example is new to the online version of the paper. It substitutes the erroneous claim found in the TPLP version of the paper.}:
$$
\ba{l}
\ar |x < 12|.\\
\ar |x > 10|.
\ea
$$
This program has no answer sets, but has a weak answer set, $\emptyset$.
Arguably, weak answer sets exhibit an {\em agnostic} attitude toward the values of variables associated with constraints that have no corresponding constraint atoms occurring in the answer sets.

\section{The {\ezcsp} Language}\label{sec:ezcsplang}
The origins of the constraint answer set solver {\ezcsp} and 
of its language go back to the development of an 
 approach for integrating ASP and constraint programming, 
in which ASP is viewed as a specification language for constraint satisfaction
problems~\cite{bal09}. In this approach, (i) ASP programs are written in such a way that some of their rules, and corresponding atoms found in their answer sets, encode the
desired constraint satisfaction problems; (ii) both the answer sets and the solutions to the
constraint problems are computed with arbitrary off-the-shelf solvers. This is achieved by an architecture that treats the underlying solvers as black boxes and relies on translation procedures for linking the ASP solver to constraint solver. The translation procedures extract from an answer set of an ASP program the constraints that must be satisfied and translate them into a constraint problem in the input language of the
corresponding constraint solver. At the core of the {\ezcsp} specification language is relation $required$, which is
used to define the atoms that encode the constraints of the constraint satisfaction problem. 

We start this section by defining the notion of propositional ez-programs and
introducing their semantics via a simple mapping into  CA programs under weak answer set semantics.  
Then, we move to describing the full language available to CASP practitioners in the {\ezcsp} system.
The tight relation between ez-programs and CA programs
makes the following evident: although the origins of {\ezcsp} are rooted in providing
a simple, yet effective framework for modeling constraint satisfaction problems, the
{\ezcsp} language developed into a full-fledged constraint answer set programming
formalism.  This also yields another interesting observation: constraint
answer set programming can be seen as a  declarative modeling framework
utilizing  constraint satisfaction solving technology.
\tcg{The MiniZinc language~\cite{mar08a} is another remarkable effort
  toward a declarative modeling framework supported by the constraint
  satisfaction technology.  It goes beyond the scope of this paper
  comparing the expressiveness of the constraint answer set
  programming and MiniZinc.} 


\paragraph{Syntax}
An {\em ez-atom} is an expression of the form
\[
required(\beta),
\]
where $\beta$ is an atom. 
Given  an alphabet $\cC$, the corresponding alphabet of {\em ez-atoms}  $\cC^{\ez}$ is obtained in a straightforward way. 
For instance, from an  alphabet $\cC_1=\{|x<12|,~  |x\geq 12|\}$
we obtain  
$\cC^{\ez}_1=\{required(|x<12|),~  required(|x\geq 12|)\}$.

A {\em (propositional) ez-program} is a tuple \[\langle E,\cA,\cC,\gamma,D \rangle,\] where
\begin{itemize}
\item $\cA$ and $\cC$ are alphabets so that
$\cA$, $\cC$, $\cC^{\ez}$ do not share the  elements,
\item $E$ is a regular logic program so that $\At(E)=\cA\cup\cC^\ez$ \tcg{and atoms from $\cC^\ez$ only occur in the head of its rules},
\item $\gamma$ is  a function from $\cC$ to constraints, and
\item $D$ is a domain.
\end{itemize}

\paragraph{Semantics}
We  define the semantics of ez-programs via a mapping to CA programs under the weak answer set semantics. 
Let  $\cE=\langle E,\cA,\cC,\gamma,D \rangle$ be an ez-program. 
By $\cP_{\cE}$ we denote the CA program  \[\langle \Pi,\cC, \gamma,D \rangle,\] 
where $\Pi$ extends $E$ by two denials
\beq
\ba{l}
\ar required(\beta),\ not\ \beta\\
\ar not\ required(\beta),\ \beta
\ea
\eeq{denialequired}
for every ez-atom $required(\beta)$ occurring in $E$.\footnote{Formula~\eqref{denialequired} is an extension of the corresponding formula from the TPLP version of the paper, which only included the first of the two denials. The latest definition of the semantics of ez programs coincides with the semantics of these programs introduced in~\cite{bal09a}. The proof of this claim can be obtained in a straightforward way from the definition of reduct and its minimal models.} 
For   a  set $X$  of atoms over
\hbox{$\At(E)\cup \cC$}
\tcb{and an evaluation $\alpha$  from the set $\cV_{\cP_{\cE}}$ of variables  to the set $D$ of values},
we say that 
\begin{itemize}
\item $X$
 is an {\em answer set} of 
$\cE$
 if $X$ is a {\em weak answer set} of $\cP_{\cE}$;
\item  a pair $\langle X, \alpha \rangle$ is an {\em extended
   answer set} of $\cE$ if $\langle M, \alpha\rangle$ is an extended weak answer set of $\cP_{\cE}$.
\end{itemize}

\begin{example}\label{ex:ezprog1}
We now illustrate the concept of an ez-program on our running example of the ``light domain''.
Let~$\cA_1$ denote the alphabet $\{switch,~lightOn,~am\}$.
 Let $E_1$ be a collection of rules 
\begin{equation}\label{ex:acpez}
\ba l
  \{switch\}.\\
  lightOn\ar\ switch, not\ am.\\
  \ar not\ lightOn.\\ 
  \{am\}.\\
  required(|x \geq 12|)\ar not\ am.\\
  required(|x<12|) \ar am. \\
\ea
\end{equation}
where $\cC^{ez}_1$ forms an alphabet of ez-atoms.
Let $\cE_1$ be an ez-program
\beq
\langle E_1,\cA_1,\cC_1,
\gamma_1,D_1 \rangle.
\eeq{eq:sample2}
The first member of the quadruple $\cP_{\cE}$ is composed of the rules from~\eqref{ex:acpez} and of the denials
\begin{equation}\label{ex:acpez2}
\ba l
 \ar required(|x \geq 12|),\ not\ |x \geq 12|.\\
 \ar required(|x<12|),\ not\ |x<12|.\\
 \ar not\ required(|x \geq 12|),\ |x \geq 12|.\\
 \ar not\ required(|x<12|),\ |x<12|.

\ea
\end{equation}
Ez-program $\cE_1$ has one answer set
\[
N_1=\{switch, lightOn,  required(|x \geq 12|), |x\geq 12|)\}\\
\]
Pairs
\beq 
\langle N_1, x=12 \rangle 
\eeq{eq:sample1}
and
$ \langle N_1, x=23 \rangle 
$
are two among twelve extended answer sets of ez-program~$\cE_1$.
\end{example}

At the core of the \ezcsp system is its \emph{solver} algorithm
(described in Section~\ref{sec:ezcspsolver}), which takes as an input a propositional ez-program and computes its answer sets. 
In order to allow for more compact specifications, the {\ezcsp} system supports an extension of
the language of propositional ez-programs, which we call {\ez}.
The language is described by means of examples next. Its definition can be found in~\ref{sec:ez-language}. Also, the part of formalization of the Weighted Sequence domain presented in Section~\ref{sec:domain} illustrates the use of the so called reified constraints, which form an important modeling tool of the {\ez} language.

\begin{example}\label{ex:ezcspprog1}
In the {\ez} language,
the ez-program $\cE_1$ introduced in ${\mathit Example}$~\ref{ex:ezprog1}  is specified as follows: 
\[
\ba l
  cspdomain(fd).\\
  cspvar(x,0,23).\\
  \{switch\}.\\
  lightOn\ar\ switch, not\ am.\\
  \ar not\ lightOn.\\ 
  \{am\}.\\
  required(x \geq 12)\ar not\ am.\\
  required(x<12) \ar am. \\
\ea
\]
The first rule specifies domain of possible csp-abstractions, which in
this case is that of finite-domains. The second rule states that $x$
is a variable over this domain ranging between $0$ and $23$. The rest
of the program follows the lines of \eqref{ex:acpez} almost verbatim. 

It is easy to see that denial~\eqref{denialequired} poses the restriction on the form of the answer sets of ez-programs so that an atom of the form $required(\beta)$ appears in an answer set if and only if  an atom of the form $\beta$ appears in it. Thus, when the \ezcsp system computes answer sets for the \ez programs, it omits $\beta$ atoms. For instance, for the  program of this example \ezcsp will output:
$$\{ cspdomain(fd), cspvar(x,0,23), required(x \geq 12), switch, lightOn, x=12  \}$$
to encode extended answer set~\eqref{eq:sample1}. 
\end{example}
\begin{example}
The {\ez} language includes support for a number of commonly-used global constraints, such as $all\_dif\!ferent$ and
$cumulative$ (more details in~\ref{sec:ez-language}). For
example, a possible encoding of the classical ``Send$+$More$=$Money''
problem is: 
\[
\begin{array}{l}
cspdomain(fd)\ldotp \\
cspvar(s,0,9)\ldotp\ \ cspvar(e,0,9)\ldotp \ \ldots\ cspvar(y,0,9)\ldotp \\

required(s*1000 + e*100 + n*10 + d + \\
\hspace*{.63in} m*1000 + o*100 + r*10 + e = \\
\hspace*{.63in} m*10000 + o*1000 + n*100 + e*10 + y)\ldotp \\

required(s \neq 0)\ldotp \ \ \ 
required(m \neq 0)\ldotp \\

required(all\_dif\!ferent([s,e,n,d,m,o,r,y]))\ldotp
\end{array}
\]
As before, the first rule specifies the domain of  possible
csp-abstractions. The next set of rules specifies the variables and
their ranges. The remaining rules state the main constraints of the
problem. Of those, the final rule encodes an $all\_dif\!ferent$
constraint, which informally requires all of the listed variables to
have distinct values. The argument of the constraint is an extensional
list of the variables of the CSP. An extensional list is a list that
explicitly enumerates all of its elements. 

A simple renaming of the variables of the problem allows us to
demonstrate the intensional specification of lists: 
\[
\begin{array}{l}
cspdomain(fd)\ldotp \\
cspvar(v(s),0,9)\ldotp\ \ cspvar(v(e),0,9)\ldotp \ \ldots\ cspvar(v(y),0,9)\ldotp \\

required(v(s)*1000 + v(e)*100 + v(n)*10 + v(d) + \\
\hspace*{.63in} v(m)*1000 + v(o)*100 + v(r)*10 + v(e) = \\
\hspace*{.63in} v(m)*10000 + v(o)*1000 + v(n)*100 + v(e)*10 + v(y))\ldotp \\

required(v(s) \neq 0)\ldotp \ \ \ 
required(v(m) \neq 0)\ldotp \\

required(all\_dif\!ferent([v/1]))\ldotp
\end{array}
\]
The argument of the global constraint in the last rule is intensional
list $[v/1]$, which is a shorthand for the extensional list, $[v(d),
  v(e), v(m), v(n), \ldots]$, of all variables of the form~$v(\cdot)$.\end{example}

\begin{example}\label{ex:riddle2}
Consider a riddle:
\begin{quote}
There are either 2 or 3 brothers in the Smith family.
There is a 3 year difference between one brother and the next
(in order of age) for all pairs of brothers.
The age of the eldest brother is twice the age of the youngest.
The youngest is at least 6 years old.
\end{quote}
Figure~\ref{fig:ex:riddle2} presents the
 \ez program that captures the riddle\footnote{The
  reader may notice that the program features the use of arithmetic
  connectives both within terms and as full-fledged
  relations. Although, strictly speaking, separate connectives should
  be introduced for each type of usage, we abuse notation slightly and
  use context to distinguish between the two cases.}. We refer to this program as~$P_1$. 
\begin{figure}
\fbox{
$
\begin{array}{l}
\mbox{\emph{\% There are either 2 or 3 brothers in the Smith family.}}\\
num\_brothers(2) \ar \lpnot num\_brothers(3)\ldotp\\
num\_brothers(3) \ar \lpnot num\_brothers(2)\ldotp\\[.1in]
index(1)\ldotp\ \ index(2)\ldotp\ \ index(3)\ldotp\\[.1in]
is\_brother(B) \ar
\hspace*{.1in}index(B), 
\hspace*{.1in} index(N),
\hspace*{.1in}num\_brothers(N),
\hspace*{.1in}B \leq N.\\[.1in]
eldest\_brother(1)\ldotp
\\
youngest\_brother(B) \ar 
\hspace*{.1in}index(B),
\hspace*{.1in}num\_brothers(B)\ldotp\\[.1in]
cspdomain(fd)\ldotp\\[.1in]
cspvar(age(B),1,80) \ar index(B), is\_brother(B)\ldotp\\[.1in]
\mbox{\emph{\% 3 year difference between one brother and the next.}}\\
required(age(B1) - age(B2) = 3)) \ar \\
\hspace*{.5in}index(B1), index(B2),
\hspace*{.1in}is\_brother(B1), is\_brother(B2),
\hspace*{.1in}B2 = B1 + 1.\\[.1in]
\mbox{\emph{\% The eldest brother is twice as old as the youngest.}}\\
required(age(BE) = age(BY) * 2) \ar\\ 
\hspace*{.5in}index(BE), index(BY), 
\hspace*{.1in}eldest\_brother(BE), 
\hspace*{.1in}youngest\_brother(BY)\ldotp\\[.1in]
\mbox{\emph{\% The youngest is at least 6 years old.}}\\
required(age(BY) \geq 6) \ar 
\hspace*{.1in}index(BY),
\hspace*{.1in}youngest\_brother(BY).
\end{array}
$}
\caption{The \ez program for the riddle of ${\mathit Example}$~\ref{ex:riddle2}\label{fig:ex:riddle2}}
\end{figure}
Note how this program contains non-constraint variables
 $B$, $N$, $B1$, $B2$, $BE$, and $BY$.
As explained in~\ref{sec:ez-language}, the grounding process
that occurs in the {\ezcsp} system transforms 
these rules into propositional (ground) rules using the same approach commonly applied to ASP\ programs. 
For instance, the last rule of program~$P_1$
 results in three ground rules
\[
\begin{array}{l}
required(age(1) \geq 6) \ar
\hspace*{.1in}index(1),
\hspace*{.1in}
youngest\_brother(1).\\
required(age(2) \geq 6) \ar 
\hspace*{.1in}index(2),
\hspace*{.1in}youngest\_brother(2).\\
required(age(3) \geq 6) \ar 
\hspace*{.1in}index(3),
\hspace*{.1in}youngest\_brother(3).\\
\end{array}
\]

The ez-program that corresponds to $P_1$  has a unique extended answer set
\[
\begin{array}{l}
\langle \{ num\_brothers(3), \\
\ \ \ \ cspvar(age(1),1,80), \ldots, cspvar(age(3),1,80), \ldots \},\\
\ \ \{ (age(1)=12, age(2)=9, age(3)=6 \} \rangle.
\end{array}
\]
The extended answer set states that there are $3$ brothers, of age $12$, $9$, and $6$
respectively. 

\end{example}

\section{Satisfiability Modulo Theories and its Integration Schemas}\label{sec:integration}
We are now ready to draw a parallel between constraint answer set
programming and  satisfiability modulo theories. To do so, we first define the SMT problem by following the lines of~\cite[Section 3.1]{nie06}.
A {\em theory} $T$ is a set of closed first-order formulas. 
A CNF formula~$F$ (a set of clauses) over a fixed finite set of ground (variable-free) first-order atoms is {\em
  $T$-satisfiable} if there exists an interpretation, in first-order sense, that satisfies  every formula in set $F\cup T$. 
Otherwise, it is called $T$-unsatisfiable.
Let $M$ be a set of ground literals. 
We say that  $M$ is a
$T$-model of~$F$ if 
\begin{enumerate}
\labitem{(m1)}{m.1} $M$ is a model of $F$ and 
\labitem{(m2)}{m.2} 
$M$, seen as a conjunction of its elements, is  {\em  $T$-satisfiable}. 
\end{enumerate}
The SMT problem for a theory $T$ is the problem of
determining, given a formula $F$, whether~$F$ has a $T$-model.
It is easy to see that in the CASP problem, $\Pi[\cC]$ in
condition~\ref{as.1}  plays the role of~$F$ in~\ref{m.1}  in the SMT 
problem. At the same time, condition~\ref{as.2} is similar to condition~\ref{m.2}.

\st
Given this tight conceptual relation between the SMT and CASP formalisms,
it is not surprising that  solvers stemming from these different
research areas share several design traits even though these areas
have been developing to a large degree independently (CASP being a 
younger field). 
We now review major integration schemas/methods  in 
 SMT solvers by following~\cite[Section 3.2]{nie06}.
During the review, we discuss how different CASP solvers account for one or
 another method.
This discussion allows us to systematize  design patterns of solvers
present both in SMT and CASP so that their relation
becomes clearer. Such a transparent view on  architectures  of solvers
immediately  translates  findings in one area to the other. Thus,
although  the case study conducted as part of our research uses CASP
technology only, we expect similar results to hold for
SMT, and for the construction of hybrid automated reasoning methods in
general. To the best of our knowledge there was no analogous effort -- 
thorough evaluation of effect of integration schemas on  performance of systems --  
in the SMT community.


\st
In every approach discussed, 
 a formula $F$ is treated as a satisfiability formula,
where each atom is considered as a
propositional symbol, {\em forgetting} about the theory $T$. Such a view
naturally invites an idea of {\em lazy} integration: the formula
$F$ is given to a SAT solver, if the solver determines that $F$ is
unsatisfiable then $F$ has no $T$-model. Otherwise, a
propositional model $M$ of $F$ found by the SAT solver is checked by a
specialized $T$-solver, which determines whether~$M$ is 
$T$-satisfiable. If so, then it is also a $T$-model of~$F$,
otherwise~$M$ 
is used to build a clause $C$ that precludes this assignment,
i.e., $M\not\models C$ while $F \cup C$ has a $T$-model if and only
if $F$ has a $T$-model. The SAT solver is invoked on an augmented
formula $F\cup C$. This process is repeated until the procedure finds
a $T$-model or returns unsatisfiable.  Note how in this approach two
automated reasoning systems -- a SAT solver and a specialized
$T$-solver -- interleave: a SAT solver generates ``candidate
models'' whereas  a $T$-solver tests whether these models are in
accordance with requirements specified by theory $T$. We find that it
is convenient to introduce the following terminology for the future
discussion: a {\em base} 
solver and a {\em theory} solver, where  the base solver is responsible
for generating candidate models and the {\em theory} solver  is responsible
for any additional testing required for stating whether a candidate
model is indeed a solution. In this paper we refer to lazy evaluation
as  {\bbox} to be consistent with the terminology often used in
CASP. 

\st
It is easy to see how the {\bbox} integration policy
translates to the realm of CASP. Given a CA program~$\cP$, an answer set solver serves the role
of base solver by generating answer sets of  the asp-abstraction of~$\cP$
(that are ``candidate answer sets'' for~$\cP$) and then uses a
CLP/CSP solver as a theory solver to verify whether
condition~\ref{as.2} is satisfied on these candidate answer sets. 
Originally, constraint answer set solver {\ezcsp} embraced the {\bbox} integration
approach in its design.\footnote{\cite{bal09} refers to 
{\bbox} integration of {\ezcsp} as {\em lightweight}
integration of ASP and
constraint programming.} 
To  solve a CASP problem via {\bbox} approach, {\ezcsp} offers a
user various options for  base  and  theory solvers.
Table \ref{table:solvers} shows some of the currently
available solvers. 
The variety of possible
configurations of {\ezcsp} illustrates how {\bbox} integration
provides great flexibility in choosing underlying {base} and {theory}
solving technology in addressing problems of interest.
\tcb{ In principle, this approach allows for a simple integration of constraint programming systems that use
 MiniZinc and FlatZinc\footnote{\url{http://www.minizinc.org/}.}
 as their front-end description languages. 
Implementing support for this interface is a topic of future research. 
}

\begin{table}[htpb]
\begin{center}
\begin{tabular}{|l||l|}
\hline
Base Solvers&Theory Solvers\\
\hline
{\smodels}~\cite{sim02}& {\sc SICStus Prolog}~\cite{sicstus}\\
{\clasp}~\cite{geb07}&{\bprolog}~\cite{zho12}\\
{\cmodels}~\cite{giu06}&\\
\hline
\end{tabular}
\caption{Base and theory solvers supported by {\ezcsp}}\label{table:solvers}
\end{center}
\end{table}

\st
The Davis-Putnam-Logemann-Loveland (DPLL) procedure~\cite{dav62} is a
backtracking-based search algorithm for deciding the satisfiability of
a propositional CNF formula. DPLL-like procedures form the basis for
most modern SAT solvers as well as answer set solvers. If a DPLL-like
procedure underlies a base solver in the SMT and CASP tasks then it
opens a door to several refinements of {\bbox} integration. We now
describe these refinements.

\st
In the {\bbox} integration approach a base solver is
invoked iteratively. Consider the SMT task: 
a CNF formula $F_{i+1}$ of the $i+1^{\hbox{th}}$ iteration to a SAT solver
consists of a CNF formula $F_{i}$ of the $i^{\hbox{th}}$ iteration and an
additional clause (or a set of clauses). 
Modern DPLL-like solvers commonly implement such
technique as {\em incremental} solving. For instance,
incremental SAT-solving allows the user to solve several SAT problems
$F_1,\dots,F_n$ one after another (using a single invocation
of the solver), if $F_{i+1}$
results from $F_i$ by adding clauses. In turn, 
the solution to~$F_{i+1}$ may benefit 
from the knowledge obtained
during solving  $F_1,\dots,F_i$.
Various modern SAT-solvers, including {\sc minisat} \cite{minisat,minisat-manual},
  implement interfaces for incremental 
SAT solving.
Similarly, the answer set solver {\cmodels} implements an interface
that  allows the user to solve several ASP problems
$\Pi_1,\dots,\Pi_n$ one after another, if $\Pi_{i+1}$
results from $\Pi_i$ by adding a set of denials.
It is natural to utilize
incremental {\sc dpll}-like procedures for enhancing the {\bbox} integration
protocol: we call this refinement
{\gbox} integration. In this approach, rather than invoking a base
solver from scratch, an incremental interface provided by a solver is
used to implement the iterative process. CASP solver {\ezcsp}
implements {\gbox} integration using the above mentioned
incremental interface by {\cmodels}.

\st
Nieuwenhuis et al.~(\citeyear{nie06}) also review such integration techniques
used in SMT
as {\em on-line SAT solver} and {\em theory propagation}.
We refer to on-line SAT solver integration as {\cbox} here. In this approach, the $T$-satisfiability of the ``partial''
assignment is checked, while the assignment is being
built by the DPLL-like procedure. This can be done fully eagerly as
soon as a change in the partial assignment occurs, or with a certain frequency, for instance at some regular
intervals. Once the inconsistency is detected, the SAT
solver is instructed to backtrack.  The theory propagation
approach extends the {\cbox} technique by allowing
a theory solver  not only to verify that a current
partial assignment is ``$T$-consistent`` but also to detect literals in
a CNF formula that must hold given the current partial assignment.  

The CASP solver  {\clingcon} exemplifies the implementation of 
the theory propagation 
integration schema in CASP. It utilizes answer set solver {\sc clasp}
as the base solver and constraint processing system {\sc
  gecode}~\cite{gecode} as the theory solver. The
 {\acsolver} and {\idp} systems are other CASP solvers that implement 
the theory propagation integration schema.
In the scope of this work, the CASP solver {\ezcsp} was extended to
implement 
the {\cbox} integration schema using {\cmodels}.
It is worth noting that all of the above approaches consider
the theory solver as a black box, disregarding its internal
structure and only accessing it through its external API.
To the best of our knowledge, no systematic investigation exists
of integration schemas that also take advantage of the internal 
structure of the theory solver.

\tcg{An important point is due here. Some key details about the \gbox
  and \cbox integration schemas have been omitted in the presentation
  above for simplicity. To make these integration schemas perform
  efficiently, learning -- a sophisticated solving technique stemming
  from SAT~\cite{zha01} -- is used to capture the information
  (explanation) retrieved due to necessity to backtrack upon theory
  solving. This information is used by the base solver to avoid
  similar conflicts.  
Section~\ref{sec:abstractsolving} presents the details on the integration schemas formally and points at the key role of learning.}


\section{The {\ezcsp} Solver}\label{sec:ezcspsolver}

In this section, we describe an algorithm for computing answer sets of CA programs. 
A specialization of this algorithm to ez-programs is used in the {\ezcsp} system.
For this reason, we begin by giving an overview of the architecture of the {\ezcsp} system. We then describe the solving algorithm.

\subsection{Architecture}
\begin{figure}[htbp]
\begin{center}
\includegraphics[clip=true,trim=0 170 0 0,width=1\columnwidth]{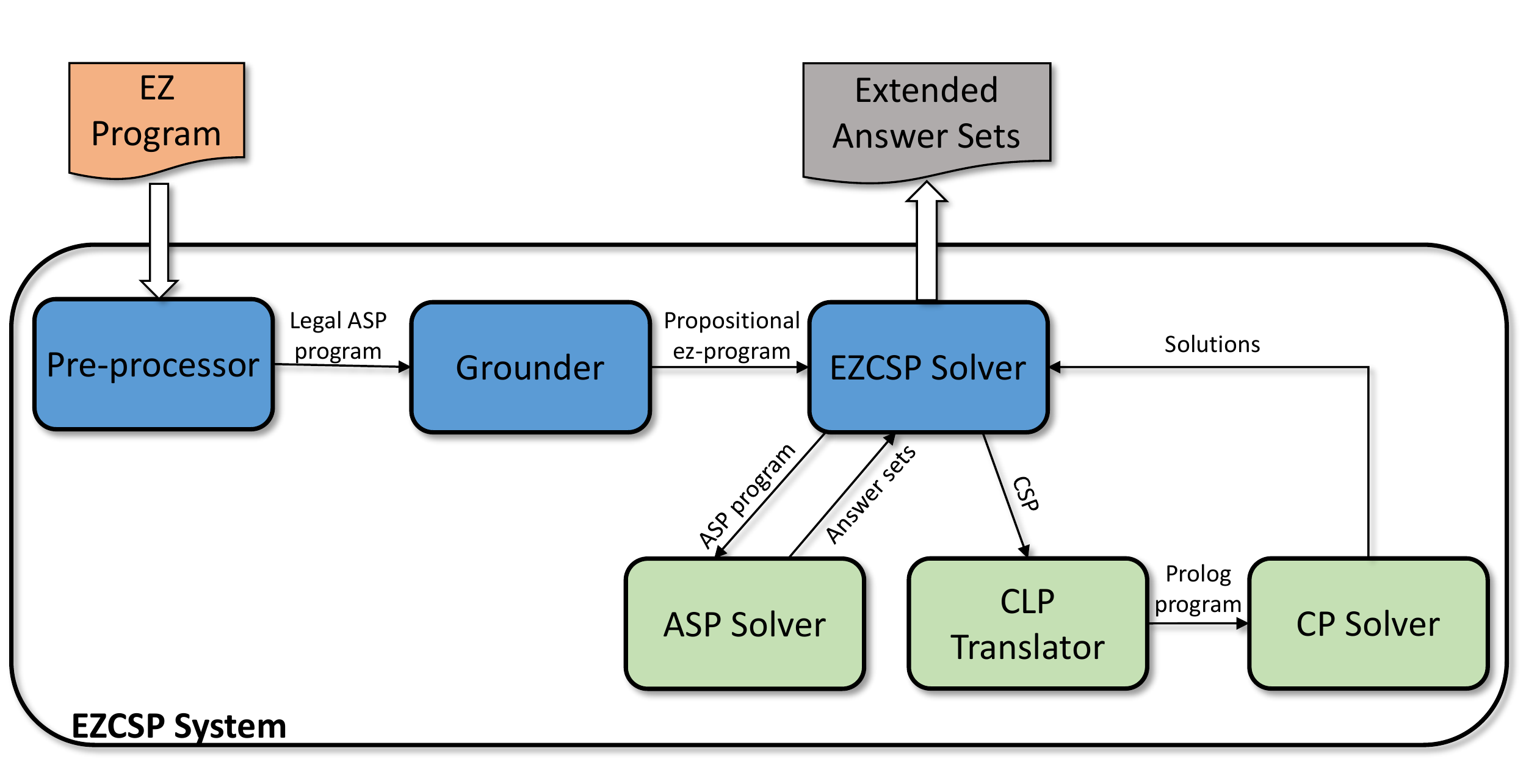}
\caption{Architecture of the {\ezcsp} system}\label{fig:arch}
\end{center}
\end{figure}
Figure~\ref{fig:arch} depicts the architecture of the system, while the narrative below elaborates on  the essential details.
Both are focused on the functioning of the \ezcsp system while
employing the \bbox integration schema. 

The first step of the execution of \ezcsp (corresponding to the \emph{Pre-processor} component in the figure) consists in running a {pre-processor, which} 
translates an input {\ez} program  into a syntactically legal ASP program.
This is accomplished by replacing the occurrences of arithmetic functions and operators
in $required(\beta)$ atoms by auxiliary function symbols. For example, an atom
$required(v>2)$ is replaced by $required(gt(v, 2))$. A similar process is also
applied to the notation for the specification of lists. For instance, an atom $required(all\_dif\!ferent([x,y]))$
is translated into $required(all\_dif\!ferent(list(x,y)))$.
The \emph{Grounder} component of the architecture transforms the
resulting program into its propositional equivalent, a regular
program, using an off-the-shelf {grounder} such as
{\gringo}~\citep{geb07b}. 
This regular program is then passed to the {\emph{{\ezcsp} Solver} component}.

The \emph{{\ezcsp} Solver} component iterates ASP and constraint
programming computations by invoking the corresponding components of
the architecture. Specifically, the \emph{ASP\ Solver} component
computes an answer set of the regular program using an off-the-shelf 
{ASP solver}, such as {\cmodels} or {\clasp}.\footnote{The ASP solver to be used can be specified by command-line options.}
If an answer set is found, the {\ezcsp} solver runs the \emph{CLP
  Translator} component, which maps the csp-abstraction corresponding
to the computed answer set 
to a {Prolog program}. The program is then passed to the
\emph{CP\ Solver} component, which uses the {CLP solver} embedded in a
Prolog interpreter, e.g. SICStus or {\bprolog},\footnote{The Prolog
  interpreter is also selectable by command-line options.} to solve
the CSP instance.  
For example, for the sample program presented in ${\mathit Example}$~\ref{ex:ezcspprog1}, the \ezcsp system produces the answer set\footnote{
        For illustrative purposes, we show the {\ez} atom $required(x \geq 12)$
        in place of the ASP atom obtained from the pre-processing phase.}:
$$\{ cspdomain(fd), cspvar(x,0,23), required(x \geq 12), switch, lightOn  \}.$$
The csp-abstraction of the program with respect to this answer set 
is translated into a Prolog rule:
\[
solve([x,V_x])\  :-\ \ V_x \geq 0,\ \ V_x \leq 23,\ \ V_x \geq12,\ \ labeling([V_x]).
\]
In this case, the CLP solver embedded in the Prolog interpreter will find feasible assignments for variable $V_x$. 
The head of the rule is designed
to return a complete solution and to ensure that the variable names used in the
{\ez} program are associated with the corresponding values. The interested reader
can refer to \cite{bal09} for a complete description of the  translation process.

Finally, the \emph{{\ezcsp} Solver} component gathers the solutions to the respective csp-abstraction and combines
them with the answer set obtained earlier to form extended answer sets. Additional extended
answer sets are computed iteratively by finding other
answer sets and the solutions to the corresponding csp-abstractions.

\subsection{Solving Algorithm}\label{sec:abstractsolving}
We are now ready to present our algorithm for computing answer sets of CA programs.
In earlier work, Lierler~(\citeyear{lier14}) demonstrated how 
the CASP language {\clingcon}~\cite{geb09} as well as the essential subset
of the CASP language {\em AC} of {\acsolver}~\cite{mel08} are captured by
CA programs. Based on those results, the
algorithm described in this section can be immediately used as an
alternative to the procedures 
implemented in systems {\clingcon} and {\acsolver}.

Usually, software systems are described by means of pseudocode. The fact
that {\ezcsp} system follows  an ``all-solvers-in-one" philosophy 
combined with a variety of integration schemas complicates the task of
describing it in this way. For example, one
configuration of {\ezcsp} may invoke answer set solver {\clasp} via
{\bbox} integration for
enumerating answer sets of an asp-abstraction of CA program, whereas another
may invoke {\cmodels} via {\gbox} integration for the same task. 
Thus, rather than committing
ourselves to a pseudocode description, we follow a path
 pioneered by Nieuwenhuis et al.~(\citeyear{nie06}). In their work, the
 authors devised a graph-based 
abstract framework for describing backtrack search procedures for
Satisfiability and SMT.
Lierler~(\citeyear{lier14}) designed a similar abstract framework that
 captures the {\ezcsp} algorithm in two cases: 
(a) when
{\ezcsp} invokes answer set solver {\smodels} via
{\bbox} integration  for enumerating answer sets of asp-abstraction program, and
(b) when
{\ezcsp} invokes answer set solver {\clasp} via
{\bbox} integration. 

In the present paper we introduce
 a graph-based abstract framework that is well suited  for capturing the
similarities and differences of the various configurations of
{\ezcsp} stemming from 
different integration schemas. 
The graph-based representation also allows us to speak of termination and correctness of procedures supporting these configurations. 
\tcb{In this framework,  nodes of a graph representing a solver
  capture its possible ``states of computation'', while edges describe
  the possible transitions from one state to another.}   
It should be noted that the graph representation is too high-level to capture some
specific features of answer set solvers or constraint
programming tools used within different {\ezcsp} configurations. For example, the graph incorporates no information on the heuristic used to select a literal upon which a decision needs to be made. This is not an issue, however:
stand alone answer set solvers have been analyzed
and compared theoretically in
the literature~\cite{angejasc06a},~\cite{giu08}~\cite{lier11} as well as empirically
in biennial answer set programming
competitions~\cite{aspcomp1},~\cite{aspcomp2},~\cite{aspcomp3}.
At the same time, {\ezcsp} treats  constraint
programming tools as ``black-boxes'' in all of its configurations. 

\subsubsection{Abstract \ezcsp}
Before  introducing  the transition system (graph) capable of
capturing a variety of \ezcsp procedures, we start by developing  some
required terminology. To make this section more self-contained we also restate some notation and definitions from earlier sections.
Recall that for a set $M$ of literals, by~$M^{+}$   we denote  the set of
positive  literals in $M$. For a CA program $\cP=\langle \Pi,\cC,\gamma,D \rangle$, a consistent and complete set $M$ of literals over $\At(\Pi)$ 
is an {\em answer set} of 
$\cP$ if  
\begin{enumerate}
        \labitem{(a1)}{as.1.REPEAT}  $M^{+}$ is an answer set of $\Pi[\cC]$ and
        \labitem{(a2)}{as.2.REPEAT}  the constraint satisfaction problem $K_{\cP,M}$
        has a solution.
\end{enumerate}
As noted in Section~\ref{sec:prelim} we can view denials as clauses. Given a denial $G$, by $G^{cl}$ we will denote a clause that corresponds to $G$, e.g.,  $(\ar not \ pm)^{cl}$ denotes a clause $pm$. We may sometime abuse the notation and refer to a clause as if it were a denial. For instance, a clause $pm$ may denote a denial $\ar not \ pm$. 

We now introduce notions for CA programs that parallel ''entailment'' for the case of classical logic formulas. 
Let $\cP=\langle \Pi,\cC,\gamma,D \rangle$  be a CA program. 
We say that $\cP$ 
\emph{asp-entails} a denial $G$ over~$\At(\Pi)$ when 
for every complete and consistent set $M$ of literals over $\At(\Pi)$ such that
\hbox{$M^{+}$} is an answer set of~$\Pi[\cC]$,~$M$ satisfies~$G^{cl}$. 
In other words, a denial is asp-entailed  if any set of literals that satisfies the condition~\ref{as.1.REPEAT} of the answer set definition is such that it satisfies this denial.
CA program~$\cP$ 
\emph{cp-entails} a denial $G$  over~$\At(\Pi)$ when 
(i) for every answer set $M$ of~$\cP$,~$M$ satisfies~$G^{cl}$
 and (ii)
there is a complete and consistent set $N$ of literals over $\At(\Pi)$ such that~$N^{+}$ is an answer set of~$\Pi[\cC]$ and $N$ does not satisfy~$G$. 
Notice that if a denial $G$ is such that a CA program~$\cP$ cp-entails~$G$, then~$\cP$ does not asp-entail $G$.
We say that~$\cP$ 
\emph{entails} a denial $G$ when $\cP$ either asp-entails or cp-entails $G$.
For a consistent set $N$ of literals over  $\At(\Pi)$ and   a literal~$l$,
we say that~$\cP$ 
\emph{asp-entails}  $l$  with respect to $N$, 
if  for every complete and consistent set $M$ of literals over~$\At(\Pi)$ such that \hbox{$M^{+}$}  is an answer set
of~$\Pi[\cC]$ and $N\subseteq M$, $l\in M$. 

\begin{example}
Recall program $\cP_1=\langle \Pi_1,\cC_1,\gamma_1,D_1 \rangle$ from ${\mathit Example}$~\ref{example:p1}. It is easy to check that
denial $\ar not\ lightOn.$  (or,  in other words  clause $lightOn$) is asp-entailed by $\cP_1$.
Also, $\cP_1$ asp-entails literals $switch$ and $\neg am$
 with respect to set $\{lightOn\}$ (and also with respect to~$\emptyset$).

Let regular program $\Pi_2$ extend program $\Pi_1$ from ${\mathit Example}$~\ref{example:p1}  by rules 
$$
\ba{l}
\{pm\}.\\
\ar\ not\ pm,\ |x\geq 12|.\\
\ar\  |x< 12|.\\
\ea
$$
Consider a CA program $\cP_2$ that differs from $\cP_1$ only by
substituting its first member $\Pi_1$ of quadruple $\langle
\Pi_1,\cC_1,\gamma_1,D_1 \rangle$ by $\Pi_2$. Denial $\ar not\ pm$ (or clause $pm$)  is
cp-entailed by~$\cP_2$. Indeed, the only answer set of this program is $\{pm, \neg |x<12|, |x\geq 12|\}$. This set satisfies 
$(\ar not\ pm)^{cl}$, in other words, clause $pm$. Consider set  
$\{\neg pm, \neg |x<12|, |x\geq 12|\}$ that does not satisfy clause $pm$. Set of atoms 
 $\{\neg pm, \neg |x<12|, |x\geq 12|\}^{+}=\{|x\geq 12|\}$ 
is an answer set of $\Pi_2[\cC_1]$. 
\end{example}
For a CA program $\cP=\langle \Pi,\cC,\gamma,D \rangle$ and a set $\Gamma$ of denials, 
by $\cP[\Gamma]$ we denote the CA program $\langle \Pi\cup\Gamma,\cC,\gamma,D \rangle$.  
  The following propositions capture important properties underlying the introduced entailment notions.
\begin{proposition}\label{lem:asp-entail}
For a CA program $\cP=\langle \Pi,\cC,\gamma,D \rangle$
 and a set $\Gamma$ of denials over $\At(\Pi)$ if $\cP$ asp-entails every denial in~$\Gamma$ then 
 (i)
programs $\Pi[\cC]$ and $(\Pi\cup\Gamma)[\cC]$
 have the same answer sets;
 (ii) CA programs~$\cP$ and $\cP[\Gamma]$ 
have the same answer sets.
\end{proposition}
\begin{proof}
We first show that condition (i) holds. From Theorem~\ref{prop:constraints} and the fact that $\cP$ asp-entails every denial in $\Gamma$ it follows that
programs $\Pi[\cC]$ and $(\Pi\cup\Gamma)[\cC]$ have the same answer sets. 
Condition (ii) follows from (i) and the fact that $K_{\cP,M}=K_{\cP[\Gamma],M}$ for any answer set $M$ of $\Pi[\cC]$ (and, consequently, for $(\Pi\cup\Gamma)[\cC]$). 
\end{proof}


\begin{proposition}\label{lem:cp-entail}
For a CA program $\cP=\langle \Pi,\cC,\gamma,D \rangle$ and a set $\Gamma$ of denials  over $\At(\Pi)$
if $\cP$ cp-entails every denial in $\Gamma$ then
CA programs $\cP$ and $\cP[\Gamma]$ 
have the same answer sets.
\end{proposition}
\begin{proof}
Let $\cP$ be a CA program $\langle \Pi,\cC,\gamma,D \rangle$. 
It is easy to see that (a) $\Pi[\cC]\cup \Gamma=(\Pi\cup \Gamma)[\cC]$ and (b) $K_{\cP,M}=K_{\cP[\Gamma],M}$.

Right-to-left:
Take  $M$ to
  be an answer set of
$\cP$. By the definition of an answer set, (i)  $M^{+}$ is an answer set of $\Pi[\cC]$ and
(ii)  the constraint satisfaction problem $K_{\cP,M}$ has a solution.
Since $\cP$ cp-entails every denial in $\Gamma$, we conclude that $M$ is a model of $\Gamma^{cl}$.
By  Theorem~\ref{prop:constraints}, $M^{+}$ is an answer set of $\Pi[\cC]\cup \Gamma$. 
From~(a) and (b) we derive that $M$ is an answer set of $\cP[\Gamma]$.

Left-to-right:
Take  $M$ to
  be an answer set of $\cP[\Gamma]$. 
  By the definition of an answer set, (i)~$M^{+}$ is an answer set of $(\Pi\cup\Gamma)[\cC]$ and
(ii)  the constraint satisfaction problem $K_{\cP[\Gamma],M}$ has a solution.
From (i) and (a) it follows that~$M^{+}$ is an answer set of $\Pi[\cC] \cup\Gamma$.
By  Theorem~\ref{prop:constraints},~$M^{+}$ is an answer set of $\Pi[\cC]$. 
By (b) and (ii) we derive that,  $M$ is an answer set of $\cP$.
\end{proof}

\begin{proposition}\label{lem:entail}
For a CA program $\cP=\langle \Pi,\cC,\gamma,D \rangle$
 and a set $\Gamma$ of denials  over $\At(\Pi)$ if $\cP$ entails every denial in $\Gamma$ then 
  (i) every answer set of $(\Pi\cup\Gamma)[\cC]$ is also an answer set of  $\Pi[\cC]$;
 (ii) CA programs $\cP$ and $\cP[\Gamma]$ 
have the same answer sets.
\end{proposition}
\begin{proof}
Condition (i) follows from Theorem~\ref{prop:constraints} and the fact that
 $(\Pi)[\cC]$ and $(\Pi\cup\Gamma)[\cC]$  only differ in denials.

We now show that condition (ii) holds.
Set $\Gamma$ is composed of two disjoint sets $\Gamma_1$ and $\Gamma_2$ (i.e., $\Gamma=\Gamma_1\cup \Gamma_2$), where $\Gamma_1$ is the set of all denials that are asp-entailed by $\cP$ and
$\Gamma_2$ is the set of all denials that are cp-entailed by $\cP$.     
By Proposition~\ref{lem:asp-entail} (ii), CA programs~$\cP$ and $\cP[\Gamma_1]$ have the same answer sets.
By Proposition~\ref{lem:cp-entail}, CA programs~$\cP[\Gamma_1]$ and $\cP[\Gamma_1\cup\Gamma_2]$ have  the same answer sets. It immediately follows that  CA programs~$\cP$ and $\cP[\Gamma_1\cup\Gamma_2]$
have the same answer sets.
\end{proof}


For an alphabet $\sigma$,
a {\sl record} relative to $\sigma$ is
a sequence~$M$ composed of \emph{distinct} literals over $\sigma$ or symbol~$\bot$, some
literals are possibly
annotated by the symbol $\dec$, which marks them as {\sl decision} literals such
that:
\begin{enumerate}
\item the set of literals in $M$ is consistent or $M=M'l$, where the set
of literals in $M'$ is consistent and 
contains $\ol{l}$, 
\item if $M=M'l^\Delta M''$, then neither $l$ nor its dual $\ol{l}$ is in $M'$, and
\item if $\bot$ occurs in $M$, then $M=M'\bot$ and $M'$ does not contain $\bot$.
\end{enumerate}
We often identify records with the set of its members disregarding
annotations. 

For a CA program~$\cP=\langle \Pi,\cC,\gamma,D \rangle$, a \emph{state} relative 
to $\cP$ is either a distinguished 
state {\fail} or a triple $M||\Gamma||\Lambda$ where $M$ is a record 
   relative to $\At(\Pi)$; 
  $\Gamma$ and $\Lambda$ are each a set of denials  that
   are entailed by~$\cP$. Given a state $M||\Gamma||\Lambda$ if neither a literal~$l$ nor~$\overline{l}$ occurs 
in $M$, then $l$ is \emph{unassigned} by the state; if ~$\bot$ does
not occur in $M$ as well as for any atom $a$ it is not the case that
both~$a$ and~$\neg a$ occur in $M$, then this state is {\em
  consistent}. 
\tcb{For a state
$M||\Gamma||\Lambda$, we call $M$,  $\Gamma$, and $\Lambda$
the {\em atomic},  {\em permanent}, and {\em temporal}  parts of the state, respectively. 
The role of the atomic part of the state is to track 
decisions (choices)  as well as inferences that the solver has made.
The permanent and temporal parts are responsible for assisting the
solver in accumulating additional information -- entailed denials by a
given program -- that becomes apparent during the search process.}

\begin{figure}
\fbox{
$
\begin{array}[t]{lll}
\hbox{{\rd}:}& 
  M||\Gamma||\Lambda ~\lrar~ 
       M~l^\dec||\Gamma||\Lambda
  & \hbox{if  $l$ is unassigned by $M$ and $M$ is consistent} \\ \ \\
\hbox{{\rfail}:}&
  M || \Gamma||\Lambda~\lrar~  {\fail}
  & \hbox{if ~} 
   \left\{ \begin{array}{l}
  \hbox{$M$ is inconsistent and}\\
  \hbox{$M$ contains no decision literals}\\   
  \end{array}\right. \\ \ \\
\hbox{{\rb}:}&
P~l^\dec~Q||\Gamma||\Lambda\lrar~ 
  P~\overline{l}||\Gamma||\Lambda
  &\hbox{if ~} 
  \left\{ \begin{array}{l}
\hbox{$P~l^\dec~Q$ is inconsistent, and}\\  
\hbox{$Q$ contains no  decision literals.}
  \end{array}\right. \\  \ \\
\hbox{{\rap}:}&
M ||\Gamma||\Lambda ~\lrar~ 
       M~l ||\Gamma||\Lambda &
  \hbox{if ~} 
 \hbox{$\cP[\Gamma\cup \Lambda]$  asp-entails $l$ with respect to $M$}  \\  \ \\

\hbox{{\rcp}:}&
M ||\Gamma||\Lambda ~\lrar~ 
       M~\bot ||\Gamma||\Lambda &
  \hbox{if  $K_{\cP,M}$ has no solution}\\  \ \\


\hbox{{\rlp}:}&
M ||\Gamma||\Lambda ~\lrar~ 
       M||\Gamma \cup\{R\}||\Lambda &
  \hbox{if  }  
\cP[\Gamma\cup\Lambda] \hbox{ entails denial $R$ \tcb{and $R\not\in\Gamma\cup\Lambda$ }}\\  \ \\ 
\hbox{{\rlt}:}&
M ||\Gamma||\Lambda ~\lrar~ 
       M||\Gamma ||\Lambda \cup\{R\}&
  \hbox{if  }\cP[\Gamma\cup\Lambda] \hbox{ entails denial $R$ \tcb{and $R\not\in\Gamma\cup\Lambda$} }\\  \ \\

\hbox{{\rpr}:}&
M ||\Gamma ||\Lambda~\lrar~ 
       \emptyset||\Gamma||\Lambda &   \hbox{if  }M\neq\emptyset\\  \ \\

\hbox{{\rcr}:}&
M||\Gamma||\Lambda ~\lrar~ 
       \emptyset||\Gamma||\emptyset &\hbox{if  }M\neq\emptyset\\  \ \\
\end{array}
$
}
\caption{The transition rules of the graph $\ezg_{\cP}$.}\label{fig:rulesez}
\end{figure}
        \begin{figure}[htbp]
\fbox{
$
                \ba{lc}
                \emptyset||\emptyset||\emptyset
                &\stackrel{\rap}\lrar\\\ \\ lightOn||\emptyset||\emptyset
                &\stackrel{\rap}\lrar\\\ \\ lightOn\ switch||\emptyset||\emptyset
                &\stackrel{\rap}\lrar\\\ \\ lightOn\ switch\ \neg am||\emptyset||\emptyset
                &\stackrel{\rap}\lrar\\\ \\ lightOn\ switch\ \neg am\ \neg |x<12|~||\emptyset||\emptyset
                &\stackrel{\rd}\lrar \\\ \\lightOn\ switch\ \neg am\ \neg |x<12|\ \neg  |x\geq 12|^\dec||\emptyset||\emptyset
                &\stackrel{\rcp}\lrar\\\ \\ lightOn\ switch\ \neg am\ \neg |x<12|\ \neg  |x\geq 12|^\dec\ \bot||\emptyset||\emptyset
                &\stackrel{\rb}\lrar\\\ \\[.1in] lightOn\ switch\ \neg am\ \neg |x<12|\ |x\geq 12|\ ||\emptyset||\emptyset&
                \ea
$
}
                \caption{Sample path in graph ${\ez}_{\cP_1}$.}\label{fig:ezpath}
         \end{figure}
  
We now define a graph {\ezg}$_{\cP}$ for  
a CA program~$\cP$.
Its nodes
are the states relative to~$\cP$. These nodes intuitively
correspond to states of computation.
The edges of the graph {\ezg}$_{\cP}$  are specified by nine
transition rules presented in Figure~\ref{fig:rulesez}.
These rules correspond to possible operations by the
{\ezcsp} system that bring it from one state of computation to another.
A path in the graph  $\ezg_\cP$
is a description of a process of search for an answer set of $\cP$.
 The process is captured via applications of transition rules.
Theorem~\ref{prop:ezcsp} introduced later in this section makes this
statement precise. 
\begin{example}\label{ex:path}
        Recall CA program $\cP_1=\langle \Pi_1,\cC_1,\gamma_1,D_1 \rangle$ introduced in ${\mathit Example}$~\ref{example:p1}. 
        Figure~\ref{fig:ezpath} presents a sample path in  ${\ez}_{\cP_1}$ with every edge annotated by the name
        of a transition rule that justifies the presence of this edge in the
        graph.
\end{example}
Now we turn our attention to an informal discussion of
 the role of each of the transition rules in $\ezg_\cP$.


\subsubsection{Informal account on transition rules}
We refer to the transition rules \rd, \rfail, \rb, \rap, \rcp
of the graph  $\ezg_\cP$  as \emph{basic}.

The unique feature of basic rules is that they only
concern the atomic part of a state. Consider a state $S=M||\Gamma||\Lambda$.
An application of any basic rule results in a state 
whose permanent and temporal parts remain unchanged, i.e., $\Gamma$
and $\Lambda$ respectively (unless it is the case of \rfail).

\paragraph{\rd}
 An application of the transition rule {\rd} to $S$ results in a state
whose   atomic part has the form
$M~l^\dec$. Intuitively this rule allows us to pursue  evaluation of
assignments that assume value of literal $l$ to be true. The fact
that this literal is marked by $\dec$ suggests that we can still reevaluate this
assumption in the future, in other words to backtrack on this decision. 

\paragraph{\rfail}
 The transition rule $\rfail$ specifies the conditions on atomic part $M$ of
state $S$ suggesting that $\fail$ is reachable from $M$. Intuitively,
if our computation brought us to such a state  transition to $\fail$
confirms that there is no solution to the problem.

\paragraph{\rb}
 The transition rule $\rb$ specifies the conditions 
   on atomic part of
the state suggesting when it is time to backtrack and what the new
atomic part of the state is after backtracking.
Rules $\rfail$ and $\rb$ share one property: they are
applicable only when states are inconsistent.

\paragraph{\rap}
  The transition rule {\rap} specifies the condition under which a new
 literal $l$ (without a decision annotation) is added to an atomic part. Such rules are commonly called {\em propagators}.
 Note that the condition   of
 \rap
\beq
\hbox{$\cP[\Gamma\cup \Lambda]$  asp-entails $l$ with respect to $M$}
\eeq{eq:asp-cond}
is defined over
 a program extended by permanent and temporal part. This fact
 illustrates the role of these entities. They carry extra
 information aquired/learnt during the computation. 
 Also  condition~\eqref{eq:asp-cond} is semantic. 
 It refers to the notion of asp-entailment, which is defined 
 by a reference to the semantics of a program.
  Propagators used by software systems typically use 
   syntactic conditions, which are easy to check by inspecting syntactic properties of a program. 
   Later in  this section we present instances of such propagators, in particular, propagators that are used within the \ezcsp solver.

\paragraph{\rcp}
 The transition rule {\rcp} specifies the condition under which
   symbol $\bot$ is added  to an atomic part. Thus it leads to a state
   that is inconsistent suggesting that the search process is either
   ready to fail or to backtrack. The condition of  {\rcp}
    $$\hbox{$K_{\cP,M}$ has no solution}$$
   represents a decision procedure that  establishes whether
   the CSP problem $K_{\cP,M}$ has solutions or not.

\bigskip
We now turn our attention to non-basic rules that concern permanent and
temporal parts of the states of computation.

\paragraph{\rlp} Recall the definition  of the transition rule {\rlp} 
\[
M ||\Gamma||\Lambda ~\lrar~ 
       M||\Gamma \cup\{R\}||\Lambda 
  \hbox{ ~~ if~~  }  
\cP[\Gamma\cup\Lambda] \hbox{  entails denial $R$ \tcb{and $R\not\in\Gamma\cup\Lambda$ }}\]
An application of this rule to a state $M ||\Gamma||\Lambda$, results in a
state whose atomic and temporal parts stay unchanged. The permanent
part is extended by a denial $R$. Intuitively the effect of this rule is
such that from this point of computation the ``permanent'' denial
 becomes effectively a part of the program being solved. This is
 essential for two reasons. First, if the learnt denial $R$ is
 cp-entailed then $\Pi \cup \Gamma\cup\Lambda$ and   $\Pi \cup
 \Gamma\cup\Lambda\cup\{R\}$ are programs with different answer
 sets. In turn, the rule $\rap$ may be applicable to some state
 $N||\Gamma \cup\{R\}||\Lambda$ and not to  $N||\Gamma||\Lambda$. 
Similarly, due to the fact that only ``syntactic'' instances of $\rap$
are implemented in solvers, the  previous statement also
holds for the case when $R$ is asp-entailed.   
 
\paragraph{\rlt} The role of the transition rule {\rlt}  is similar to that of
{\rlp}, but the learnt denials by this rule are not meant to be preserved permanently in the
computation.

\paragraph{{\rpr} \hbox{\em and} {\rcr}} The transition rule {\rpr}  allows the computation to
start from ``scratch'' with respect to atomic part of the state. The transition
rule {\rcr} forces the computation to
start from ``scratch'' with respect to not only atomic part of the state but also
all temporally learnt denials. These restart rules are essential in
understanding the key differences between various integration
strategies that are of focus in this paper.

\subsubsection{Formal properties of ${\ezg}_\cP$}
We call the state $\emptyset||\emptyset||\emptyset$ --- 
{\em initial}. We say that a node in the graph is \emph{semi-terminal} 
if  no rule other than {\rlp}, {\rlt},  {\rpr}, {\rcr} is applicable to it \tcb{ (or, in other words, if no single basic rule is applicable to it).}
\tcb{We say that a path in ${\ezg}_\cP$
is {\em restart-safe} when, prior to any  edge $e$ due to an
application of $\rpr$ or $\rcr$ on this path, there is an edge $e'$
due to an application of $\rlp$  
such that: (i) edge $e'$ precedes $e$; (ii)  $e'$ does not precede any other edge $e''\neq e$ due to  $\rpr$ or $\rcr$.
We say that a restart-safe path $t$ is {\em maximal} if (i) the first state in $t$ is an initial state, and (ii) 
 $t$ is not a subpath of any restart-safe path  $t'\neq t$. 
}
\begin{example}
                Recall CA program $\cP_1=\langle \Pi_1,\cC_1,\gamma_1,D_1 \rangle$ introduced in ${\mathit Example}$~\ref{example:p1}.  
                Trivially a sample path in  ${\ez}_{\cP_1}$ in Figure~\ref{fig:ezpath} is a  restart-safe path. 
A nontrivial  example of restart-safe path in  ${\ez}_{\cP_1}$ follows
\beq
\ba{lc}
\emptyset||\emptyset||\emptyset
&\stackrel{\rlp}\lrar\\\ \\ \emptyset||\{\ar not\ switch\}||\emptyset
&\stackrel{\rap}\lrar\\\ \\ lightOn||\{\ar not\ switch\}||\emptyset
&\stackrel{\rpr}\lrar\\\ \\ \emptyset||\{\ar not\ switch\}||\emptyset.&
\ea
\eeq{eq:restart-safe}
Similarly, a path that extends the path above as follows 
$$
\ba{lc}
&\stackrel{\rlp}\lrar\\\ \\ \emptyset||\{\ar not\ switch,\ \ar am\}||\emptyset
&\stackrel{\rap}\lrar\\\ \\ lightOn||\{\ar not\ switch,\ \ar am\}||\emptyset
&\stackrel{\rpr}\lrar\\\ \\ \emptyset||\{\ar not\ switch,\ \ar am\}||\emptyset&
\ea
$$
is restart-safe.

A simple path in  ${\ez}_{\cP_1}$ that is not restart-safe
\[
\ba{lc}
\emptyset||\emptyset||\emptyset
&\stackrel{\rap}\lrar\\ lightOn||\emptyset||\emptyset
&\stackrel{\rpr}\lrar\\ \emptyset||\emptyset||\emptyset.&
\ea
\]
Indeed, condition (i) of the restart-safe definition  does not hold.
Another example of a not restart-safe path is a path  that extends path~\eqref{eq:restart-safe} as follows 
$$
\ba{lc}
&\stackrel{\rap}\lrar\\\ \\ lightOn||\{\ar not\ switch\}||\emptyset
&\stackrel{\rpr}\lrar\\\ \\ \emptyset||\{\ar not\ switch\}||\emptyset.&
\ea
$$
Indeed, condition (ii) of the restart-safe definition  does not hold for the second occurrence of the \rpr edge.
\end{example}

The following theorem captures key properties of the graph
${\ezg}_\cP$.
They suggest that the graph can be used  for deciding whether a
program with constraint atoms has an answer set.
\begin{thm}\label{prop:ezcsp}
For any CA program $\cP$:
\begin{itemize}
\item [(a)] 
\tcb{ every restart-safe path 
 in $\ezg_\cP$ is finite, and
any  maximal restart-safe path ends with a state that is semi-terminal,
 }
\item [(b)] for any  semi-terminal state $M||\Gamma||\Lambda$  of
  $\ezg_\cP$ reachable from initial state, $M$ is an answer
  set of $\cP$,
\item [(c)] \tcb{ state {\fail} is reachable from initial state in
  $\ezg_\cP$ by a restart-safe path if and  only if  $\cP$ has no answer set.}
\end{itemize}
\end{thm}
On the one hand, part (a) of Theorem~\ref{prop:ezcsp} asserts that, if we 
construct a \tcb{restart-safe path from initial state}, then some semi-terminal state is 
eventually reached. On the other hand, parts (b) and (c)
assert that, as soon as a semi-terminal state 
is reached by following any restart-safe path, the problem of deciding
whether CA program $\cP$ has answer sets is  
solved.  Section~\ref{sec:config} describes the varying configurations
of the \ezcsp system. 
\begin{example}
Recall ${\mathit Example}$~\ref{ex:path}.
                Since the last state in the sample path presented in Figure~\eqref{fig:ezpath} is semi-terminal, Theorem~\ref{prop:ezcsp} asserts that the set of literals composed of the elements of this semi-terminal state
                forms the  answer set of~CA program $\cP_1$. Indeed, this set
                coincides with the answer set~$M_1$ of  $\cP_1$ presented in ${\mathit Example}$~\ref{ex:answerset}.      
\end{example}
In our discussion of the transition rule $\rap$ we mentioned how
the {\ezcsp} solver accounts only for some transitions due
to this rule.
Let 
$\cP=\langle \Pi,\cC, \gamma,D \rangle$ be a CA  program. 
By $\ezsmg_\cP$ we denote an edge-induced subgraph of $\ezg_\cP$,
where we drop the edges that correspond to the application of
transition rules $\rap$ not accounted by the following two transition rules (propagators)
$\rup$ and $\runf$:
\[
\hspace*{-.3in}
\begin{array}[t]{l}
\hbox{{\rup}: }
M ||\Gamma||\Lambda ~\lrar~ 
       M~l ||\Gamma||\Lambda 
  \left\{ \begin{array}{l}
\hbox{$C\vee l \in{(\Pi[\cC]\cup\Gamma\cup \Lambda)^{cl}}$,}\\ 
\hbox{$M$ is consistent,}\\
\hbox{$M\models \ol{C}$}\\  
  \end{array}\right. \\  \ \\
\hbox{{\runf}: }
M ||\Gamma||\Lambda ~\lrar~ 
       M~l ||\Gamma||\Lambda 
  \left\{ \begin{array}{l}
\hbox{$M$ is consistent, and there is  literal $l$ so that}\\
\hbox{$\ol{l}\in U$ for a set $U$, which is}\\
\hbox{unfounded on $M$ w.r.t. $\Pi[\cC]\cup\Gamma\cup\Lambda$}\\
  \end{array}\right. \\  \ \\
  \end{array}
\]
These two propagators rely on properties that can be checked by efficient procedures. 
The conditions of these transition rules are such that they are satisfied only if $\cP[\Gamma\cup\Lambda]$ asp-entails $l$ w.r.t. $M$.
In other words, the transition rules \rup or \runf are applicable only in states where \rap is applicable. The other direction is not true. 
\tcb{ Theorem~\ref{prop:ezcsp} holds if we replace $\ezg_\cP$ by
$\ezsmg_\cP$ in its statement. The proof of this theorem relies on the
statement of Theorem~\ref{prop:leone}, and is given at the end of this subsection. 
  Graph $\ezsmg_\cP$ is only one of the possible 
subgraphs of the generic graph $\ezg_\cP$ that share its key
properties stated in Theorem~\ref{prop:ezcsp}.  
These properties show that 
graph $\ezsmg_\cP$  gives rise to a class of correct algorithms
for computing answer sets of programs with constraints.
It provides a
proof of correctness  of every CASP solver in this class and a
proof of termination under the assumption that restart-safe paths are considered by a solver.}
Note how much weaker propagators, such as \rup and \runf, than $\rap$ 
are sufficient to ensure the correctness of  respective
solving procedures. 
We picked the graph $\ezsmg_\cP$ 
 for illustration as it captures the essential
propagators present in modern (constraint) answer set solvers and allows a more
concrete view on the $\ezg_\cP$ framework. Yet the goal of this work
is not to detail the variety of possible propagators of (constraint) answer set solvers but
master the understanding of hybrid procedures that include this
technology. Therefore in the rest of this section we turn our
attention back to the $\ezg_\cP$ graph and use this graph to formulate
{\bbox}, {\gbox}, and {\cbox}
 configurations of the CASP solver {\ezcsp}.

\tcb{The rest of this subsection presents a proof of
  Theorem~\ref{prop:ezcsp} as well as a proof of the similar theorem
  for the graph $\ezsmg_\cP$.} 
\begin{proof}[Proof of Theorem~\ref{prop:ezcsp}] 
{\color{black} (a)
 Let $\cP$ be a CA program $\langle\Pi,\cC,\gamma,D\rangle$.

 We first show that any path in $\ezg_\cP$ that does not contain $\rcr$ or $\rpr$ edges is finite.
 We name this statement {\em Statement 1.}
 
Consider any path $t$ in $\ezg_\cP$ that does not contain $\rcr$ or $\rpr$ edges.

For any list $N$ of literals, by $|N|$ we denote the length of $N$.
Any state $M||\Gamma||\Lambda$ has the form $M_0~l^\dec_1~M_1\dots
l^\dec_p~M_p||\Gamma||\Lambda$, where $l^\dec_1\dots l^\dec_p$ are all  decision literals of $M$;
we define $\alpha(M||\Gamma||\Lambda)$ as the sequence of nonnegative integers
$|M_0|,|M_1|,\dots,|M_p|$, and $\alpha(\fail)=\infty$.
For any two states, $S$ and $S'$, of ${\ezg}_\cP$, we understand
$\alpha(S)<\alpha(S')$ as the lexicographical order. 
We  note that, for any state $M||\Gamma||\Lambda$, 
the value of $\alpha$ is based only on the first component, $M$, of the state. 
Second, there is a finite number of
distinct values of $\alpha$ for the states of ${\ezg}_\cP$ due to the fact that there is a finite
number of distinct $M$'s  over $\cP$. 
We now define  relation \emph{smaller} over the states of ${\ezg}_\cP$.
We say that state $M||\Gamma||\Lambda$ is {\em smaller} than state
$M'||\Gamma'||\Lambda'$ when either
\begin{enumerate}
                \item $\Gamma\subset\Gamma'$, or        
            \item 
            $\Gamma=\Gamma'$, and $\Lambda\subset\Lambda'$, or       
        \item $\Gamma=\Gamma'$, $\Lambda=\Lambda'$, and $\alpha(M||\Gamma||\Lambda)<\alpha(M'||\Gamma'||\Lambda')$.       
\end{enumerate}
It is easy to see that this relation is anti-symmetric and transitive.

By the definition of the transition rules  of ${\ezg}_\cP$, 
if there is an edge from $M||\Gamma||\Lambda$ to $M'||\Gamma'||\Lambda'$ in ${\ezg}_\cP$
   formed by any basic transition rule or rules $\rlp$ or $\rlt$,
 then $M||\Gamma||\Lambda$ is  smaller than state $M'||\Gamma'||\Lambda'$.
Observe that (i)
 there is a  finite number of distinct values of $\alpha$,
 and (ii) there is a finite  number of distinct denials entailed by $\cP$.
Then, it follows that there is only a finite number of edges in $t$, and, thus, Statement 1 holds.

We call a subpath from state $S$ to state $S'$ of some path in
$\ezg_\cP$ {\em restarting} when (i) an edge that follows $S$ is due
to the application of rule $\rlp$, (ii) an edge leading to $S'$ is due
to the application of rule $\rcr$ or $\rpr$, 
and (iii) on this subpath, there are no other edges due to applications of 
$\rlp$, $\rcr$, or $\rpr$, but the ones mentioned above.
Using Statement 1, it follows that  any restarting subpath is finite.

Consider any restart-safe path $r$ in $\ezg_\cP$. We construct a path
$r'$ by dropping some finite fragments from $r$. This is accomplished
by replacing each restarting subpath of $r$ from state~$S$ to state
$S'$ by an edge from $S$ to $S'$ that we call {\em Artificial}. It is
easy to see that an edge in $r'$ due to {\em Artificial} leads from
a state of the form $M||\Gamma||\Lambda$ to a state
\hbox{$\emptyset||\Gamma\cup \{C\}||\Lambda'$}, where $C$ is a denial.
Indeed, within a restarting subpath an edge due to rule $\rlp$ occurred introducing denial $C$. 
State $M||\Gamma||\Lambda$ is smaller than the state $\emptyset||\Gamma\cup \{C\}||\Lambda'$.
At the same time, $r'$ contains no edges due to applications of $\rcr$
or $\rpr$. Indeed, we eliminated these edges in favor of edges called {\em Artificial}.
Thus by the same argument as in the proof of Statement 1,
$r'$ contains a finite number of edges. We can now conclude that $r$ is
finite. 

It is easy to see that maximal restart-safe path ends with a state that is semi-terminal. Indeed, assume the opposite: there is 
a maximal restart-safe path $t$, which ends in a non semi-terminal
state $S$. Then, some basic rule applies to state $S$. Consider path
$t'$ consisting of path $t$ and a transition due to a basic rule
applicable to $S$. Note that $t'$ is also a restart-safe path, and
that $t$ is a subpath of $t'$. This contradicts  the definition of
maximal.  
}
 
\medskip
 \noindent
(b) 
Let $M||\Gamma||\Lambda$ be a semi-terminal state so that none of the 
Basic rules are applicable. From the fact that {\rd} is not applicable, we conclude
that  $M$ assigns all literals or $M$ is inconsistent. 

We now show that $M$ is consistent. Proof by contradiction. Assume that $M$ is
inconsistent. Then, since {\rfail} is not applicable,~$M$ contains a
decision literal. Consequently,  $M||\Gamma||\Lambda$ is a state in which \rb is applicable.
This contradicts our assumption that $M||\Gamma||\Lambda$ is semi-terminal.

Also, $M^{+}$ is an answer set of $\Pi[\cC]$. Proof by contradiction. Assume that $M^{+}$ is not an answer set of $\Pi[\cC]$. 
It follows that that $M$ is not an answer set of $\cP$.
By Proposition~\ref{lem:entail}, it follows that $M$ is not an answer set of $\cP[\Gamma\cup\Lambda]$
and $M^{+}$ is not an answer of $\Pi[\cC]\cup\Gamma\cup\Lambda$.
Recall that $\cP[\Gamma\cup\Lambda]$ asp-entails a literal $l$ with
respect to $M$ if for every complete and consistent set $M'$ of
literals over $\At(\Pi)$ such that $M'^{+}$ is an answer set of
$\Pi[\cC]\cup\Gamma\cup\Lambda$ and $M\subseteq M'$, $l\in M'$. 
Since $M$ is complete and consistent set of literals over $\At(\Pi)$ 
it follows that there is no complete and consistent set $M'$ of
literals over $\At(\Pi)$ such that $M\subseteq M'$ and $M'^{+}$ is an
answer set of $\Pi[\cC]\cup\Gamma\cup\Lambda$. 
We conclude that   $\cP[\Gamma\cup\Lambda]$ asp-entails any literal $l$.
Take $l$ to be a complement of some literal occurring in $M$. 
 It follows that \rap is applicable in state $M||\Gamma||\Lambda$
 allowing a transition to state $M~l||\Gamma||\Lambda$. This
 contradicts our assumption that $M||\Gamma||\Lambda$ is
 semi-terminal. 

CSP $K_{\cP,M}$ has a solution. This immediately follows from the
application condition of the transition rule \rcp and the fact that
the state $M||\Gamma||\Lambda$ is semi-terminal. 

From the conclusions that $M^{+}$ is an answer set of $\Pi[\cC]$ and
$K_{\cP,M}$ has a solution we derive that $M$ is an answer set of
$\cP$. 
 
\medskip
\noindent
(c) We start by proving an auxiliary statement: 

\medskip
{\em  Statement 2:
 For any CA program $\cP$, and a path from an initial state to $l_1\dots
  l_n||\Gamma||\Lambda$ in $\ezg_\cP$, every answer set $X$ for $\cP$
  satisfies $l_i$ if it satisfies all decision literals $l_j^\dec$
  with $j\leq i$.  
}

\noindent
By  induction on the length of a path.
Since the property trivially holds in the initial state, we
only need to prove that all transition rules of $\ezg_\cP$ preserve it.

Consider an edge $M||\Gamma||\Lambda\lrar S$ where 
$S$ is either a fail state or state of the form $M'||\Gamma'||\Lambda'$,
$M$ is a sequence $l_1~\dots~l_k$ such that every answer set $X$  of $\cP$
satisfies $l_i$  if it satisfies all decision literals $l^\dec_j$ with
$j\leq i$.

$\rd$, {\rfail}, ${\rcp}$ {\rlp}, {\rlt}, {\rpr}, {\rcr}: Obvious.

${\rap}$: $M'||\Gamma'||\Lambda'$ is $M~l_{k+1}||\Gamma||\Lambda$.  
Take any answer set $X$  of $\cP$ such that
$X$  satisfies all decision literals $l^\dec_j$ with
$j\leq k+1$. 
From the inductive hypothesis it follows that $X$ satisfies $M$.
Consequently, $M\subseteq X$ since $X$ is a consistent and complete set of literals.  
 From the definition of {\rap}, $\cP$ asp-entails~$l_{k+1}$ with respect to $M$.
 We also know that $X^{+}$ is an answer set of $\Pi[\cC]$. Thus, $l_{k+1}\in X$. 


${\rb}$: $M$ has the form $P~l^\dec_{i}~Q$ where $Q$ contains no decision
literals. $M'||\Gamma'||\Lambda'$ has the form $P~\overline{l_{i}}||\Gamma||\Lambda$.
Take any answer set $X$  of $\cP$ such that~$X$  satisfies all decision literals $l^\dec_j$ with
$j\leq i$. 
We need to show that $X\models\overline{l_{i}}$.
By contradiction. Assume that $X\models~\!\!{l_{i}}$.
By the inductive hypothesis, 
since~$Q$ does not contain decision literals,  
it follows that~$X$ satisfies 
$P~l^\dec_i~Q$, that is, $M$. This is impossible because 
$M$ is inconsistent. Hence,  $X\models\overline{l_{i}}$.

\medskip
Left-to-right: Since {\fail} is reachable from the initial state by a \tcb{restart-safe} path,
there is an inconsistent state 
$M||\Gamma||\Lambda$ without decision literals such that there exists a path from
the initial state to $M||\Gamma||\Lambda$.  By 
Statement~2, any answer set of~$\cP$ satisfies~$M$. 
Since~$M$ is inconsistent  we conclude that~$\cP$ has no answer sets.

Right-to-left: 
\tcb{From (a) it follows that any maximal restart-safe path is a path
  from initial state to some semi-terminal state $S$.  
 By~(b), this state $S$ cannot be different from
{\fail}, because $\cP$ has no answer sets.  }
\end{proof}

\begin{thm}\label{prop:ezsmg} For any CA program $\cP$,
\begin{itemize}
\item [(a)] 
\tcb{ every restart-safe path 
 in $\ezsmg_\cP$ is finite, and
any  maximal restart-safe path ends with a state that is semi-terminal,
 }
 \item [(b)] for any  semi-terminal state $M||\Gamma||\Lambda$  of
  $\ezsmg_\cP$ reachable from initial state, $M$ is an answer
  set of $\cP$,
\item [(c)] \tcb{ state {\fail} is reachable from initial state in
  $\ezsmg_\cP$ by a restart-safe path if and  only if  $\cP$ has no answer sets.}
\end{itemize}
\end{thm}
\begin{proof}
 Let $\cP$ be a CA program $\langle\Pi,\cC,\gamma,D\rangle$.
 
(a) This part is proved as part (a) in proof of Theorem~\ref{prop:ezcsp}. 

\medskip
\noindent
(b) Let $M||\Gamma||\Lambda$ be a semi-terminal state so that none of the 
basic rules are applicable ($\rup$ and $\runf$ are basic rules). 
As in proof of part  (b) in  Theorem~\ref{prop:ezcsp} we conclude 
that  $M$ assigns all literals and is consistent. Also, CSP $K_{\cP,M}$ has a solution.

We now illustrate that, 
$M^{+}$ is an answer set of $\Pi[\cC]$. Proof by contradiction. Assume that $M^{+}$ is not an answer set of $\Pi[\cC]$. 
It follows that that $M$ is not an answer set of $\cP$.
By Proposition~\ref{lem:entail}, it follows that $M$ is not an answer set of $\cP[\Gamma\cup\Lambda]$
and $M^{+}$ is not an answer of $\Pi[\cC]\cup\Gamma\cup\Lambda$.
By Theorem~\ref{prop:leone}, 
it follows that either $M$ is not a model of
$\Pi[\cC]\cup\Gamma\cup\Lambda$ or~$M$ contains a non-empty subset
unfounded on $M$ w.r.t. $\Pi[\cC]\cup\Gamma\cup\Lambda$. In case the
former holds we derive that the rule \rup is applicable in the state
$M||\Gamma||\Lambda$.  In case the later holds we derive that the rule
\runf is applicable in the state $M||\Gamma||\Lambda$.  
This contradicts our assumption that $M||\Gamma||\Lambda$ is semi-terminal.

From the conclusions that $M^{+}$ is an answer set of $\Pi[\cC]$ and $K_{\cP,M}$ has a solution we derive that $M$ is an answer set of $\cP$.

\noindent
(c) Left-to-right part of the proof follows from Theorem~\ref{prop:ezcsp} (c, left-to-right)  and the fact that $\ezsmg_\cP$ is a subgraph of $\ezg_\cP$.

Right-to-left part of the proof follow the lines of
Theorem~\ref{prop:ezcsp} (c, right-to-left).
\end{proof}

\subsection{Integration Configurations of \ezcsp}\label{sec:config}
We can characterize the algorithm
of a specific solver that utilizes the transition rules of
the graph $\ezg_\cP$
by describing a strategy for choosing a path in this graph. 

\paragraph{\bf \bbox: }
A
configuration of {\ezcsp} that invokes an answer set solver  via
{\bbox} integration for
enumerating answer sets of an asp-abstraction program is captured by the following strategy in
navigating the graph~$\ezg_\cP$
\begin{enumerate}
\item $\rpr$  never applies,
\item rule $\rcp$ never applies to the states where one of these rules are applicable: $\rd$, $\rb$, $\rfail$, $\rap$,
\item $\rlt$ may apply anytime with the restriction that the denial $R$ learnt by the application of this rule is such that $\cP$  asp-entails $R$,
\item single application of $\rlp$ follows immediately  after an
application of the rule $\rcp$. Furthermore, the denial $R$ learnt by the application of this rule is such that $\cP$ cp-entails $R$,
\item $\rcr$ follows immediately after an
application of the rule $\rlp$. $\rcr$ does not apply under any other condition.
\end{enumerate}

It is easy to see that the specifications of the strategy above forms a subgraph of the graph $\ezg_\cP$. Let us denote this subgraph by  $\ezg^b_\cP$.  Theorem~\ref{prop:ezcsp} holds if we replace $\ezg_\cP$ by
$\ezg^b_\cP$ in its statement:
\begin{thm}\label{prop:ezgb} For any CA program $\cP$,
        \begin{itemize}
                \item [(a)] 
                \tcb{ every restart-safe path 
                        in $\ezg^b_\cP$ is finite, and
                        any  maximal restart-safe path ends with a state that is semi-terminal,
                }
                \item [(b)] for any  semi-terminal state $M||\Gamma||\Lambda$  of
                $\ezg^b_\cP$ reachable from initial state, $M$ is an answer
                set of $\cP$,
                \item [(c)] \tcb{ state {\fail} is reachable from initial state in
                        $\ezg^b_\cP$ by a restart-safe path if and  only if  $\cP$ has no answer sets.}
        \end{itemize}
\end{thm}
\begin{proof}
        Let $\cP$ be a CA program $\langle\Pi,\cC,\gamma,D\rangle$.
        
                \smallskip
                \noindent
        (a) This part is proved as part (a) in proof of Theorem~\ref{prop:ezcsp}. 
        
        \smallskip
        \noindent
        (b) Graph $\ezg^b_\cP$  is the subgraph of $\ezg_\cP$. At the same time it is easy to see that any non semi-terminal state in $\ezg_\cP$ is also a non semi-terminal state in $\ezg^b_\cP$.
        Thus, claim (b) follows from Theorem~\ref{prop:ezcsp} (b).

\smallskip
        \noindent
        (c) Left-to-right part of the proof follows from Theorem~\ref{prop:ezcsp} (c, left-to-right)  and the fact that $\ezg^b_\cP$ is a subgraph of $\ezg_\cP$.
        
        Right-to-left part of the proof follows the lines of
        Theorem~\ref{prop:ezcsp} (c, right-to-left).
\end{proof}

\paragraph{\bf \gbox: }
A configuration of {\ezcsp} that invokes an answer set solver  via
{\gbox} integration for
enumerating answer sets of asp-abstraction program is captured by the strategy in
navigating the graph~$\ezg_\cP$ that differs from the strategy of \bbox in rules 1 and 5 only.
Below we present only these rules.
\begin{itemize}
\item [1.] $\rcr$  never applies,
\item [5.] $\rpr$ follows immediately after an
application of the rule $\rlp$. $\rpr$ does not apply under any other condition.
\end{itemize}

\paragraph{\bf \cbox: }
A configuration of {\ezcsp} that invokes an answer set solver  via
{\cbox} integration for
enumerating answer sets of asp-abstraction program is captured by the following strategy in
navigating the graph~$\ezg_\cP$
\begin{itemize}
\item $\rcr$ and $\rpr$  never apply.
\end{itemize}

Similar to the \bbox case, the specifications of the \gbox and \cbox strategies  form  subgraphs of the graph $\ezg_\cP$.  Theorem~\ref{prop:ezcsp} holds if we replace $\ezg_\cP$ by
these subgraphs. We avoid stating  formal proofs as they follow the lines of proof for Theorem~\ref{prop:ezgb}.

We note that the outlined strategies provide only a skeleton of  the algorithms implemented in these systems.
Generally, any particular
configurations of {\ezcsp} can be captured by some subgraph of
$\ezg_\cP$. The provided specifications of {\bbox}, {\gbox}, and {\cbox} scenarios
allow more freedom than specific configurations of {\ezcsp} do.
 For example, in any setting of {\ezcsp} it will never
follow an edge due to the transition {\rd} when the transition {\rap}
is available. Indeed, this is a design choice of all available answer
set solvers that {\ezcsp} is based upon. The provided skeleton is
meant to highlight the essence of
key differences between the variants of integration approaches. 
For instance, it is apparent that any application of $\rcr$ forces us
to restart the search process by forgetting about atomic part of a current state as well as some previously learnt clauses. The {\bbox}
integration architecture is the only one allowing this transition.

As discussed earlier, the schematic rule $\rap$ is more informative than any real
propagator implemented in any answer set solver. These solvers are
only able to identify some literals that are asp-entailed by a
program with respect to a state. Thus if a program is extended with
additional denials a specific propagator may find additional literals
that are asp-entailed. This observation is important in
understanding the benefit that  $\rpr$ provides in comparison to
$\rcr$. Note that applications of these rules highlight the difference
between $\bbox$ and $\gbox$.

\section{Application Domains}\label{sec:domain}
In this work we compare and contrast different integration schemas of
hybrid solvers  on three application domains
that stem from various subareas of computer science:  {\sl
  weighted-sequence}~\cite{lierPadl12}, {\sl incremental scheduling}~\cite{bal11},
  {\sl reverse folding}.
The weighted-sequence domain is a handcrafted
 benchmark, whose  key features  are inspired by the important
 industrial problem of finding an optimal join
order by cost-based query optimizers in database systems. 
The problem is not only practically relevant but proved to be hard for
current ASP and CASP technology as illustrated in~\cite{lierPadl12}.
The incremental scheduling domain stems from a problem
occurring in commercial printing.  CASP
offers an elegant solution to it.  
The {\sl reverse folding} problem is inspired by VLSI design --  the process of creating an integrated circuit  by combining thousands of transistors into a single chip. 

This section provides a
brief overview of these applications. 
All benchmark domains are 
from the {\sl Third  Answer Set Programming Competition --
  2011} ({\aspcomp})~\cite{aspcomp3}, 
in particular, 
 the {\sl Model and Solve} track. 
We chose these domains for our investigation for several reasons. 
First, these problems touch on applications relevant to various
industries. Thus, studying different possibilities to model and solve
these problems is of value.
Second, each one of them displays
features that benefit from the synergy of computational methods in ASP
and CSP. Each considered problem contains variables ranging over a
large integer domain thus making grounding required in pure ASP a
bottleneck. Yet, the modeling capabilities
of ASP and availability of such sophisticated solving techniques such as learning
makes ASP attractive for designing solutions to these domains. 
As a result,  CASP languages and  solvers become a natural choice for
these benchmarks making them ideal for our investigation. 
 
\smallskip
\noindent
{\bf Three Kinds of CASP Encodings:} Hybrid languages
such as CASP 
combine constructs and processing techniques stemming from different
formalisms. As a result, depending on how the encodings are crafted,
one underlying solver may be used more heavily than the other.
For example, any ASP encoding of a problem is also a CASP
formalization of it. Therefore, the computation for such encoding relies
 entirely on the base solver and the features and performance
of the theory solver are irrelevant to it. We call this a {\em pure-ASP} encoding.
At the other end of the spectrum are {\em pure-CSP} encodings: encodings that
consist entirely of ez-atoms. From a computational perspective, such an encoding
exercises only the theory solver. (From a specification perspective,
the use of CASP is still meaningful, as it allows for a convenient,
declarative, and at the same time executable specification of the
constraints.)
In the middle of the spectrum are {\em true-CASP} encodings, which, typically,
are non-stratified and include collections of ez-atoms expressing constraints whose solution is non-trivial.

\st
An analysis of these varying kinds of encodings in CASP gives us a better
perspective on how different integration schemas are affected by the
design choices made during the encoding of a problem. At the same time
considering the encoding variety allows us to verify our
intuition that true-CASP is an appropriate modeling and solving choice
for the explored domains.    
We  conducted experiments on encodings
falling in each category for all benchmarks considered. 

\smallskip
\noindent
The {\bf weighted-sequence} (\wseq) domain is a handcrafted benchmark 
problem.
Its key features are inspired by the important 
industrial problem of finding an optimal join order by cost-based query
optimizers in database systems. 
\citeauthor{lierPadl12}~(\citeyear{lierPadl12}) provides a
 complete description of the problem itself as well as the formalization named {\seq}
 that became the encoding used in the present paper.

\st
In the weighted-sequence problem we are given a set of leaves (nodes) and an
integer $m$ -- maximum cost. Each leaf is a pair {\em (weight, cardinality)} where
{\em weight} and {\em cardinality} are integers. Every sequence (permutation) of leaves is
such that all leaves but the first are assigned a {\em color} that, in
turn, associates a leaf with a {\em cost} (via a cost formula). 
A colored sequence is
associated with the {\em cost} that is a sum of leaves costs. 
The task is to find a colored sequence
with cost at most $m$. We refer the reader to \cite{lierPadl12}
for the details of pure-ASP encoding \seq. The same paper also
contains the details on a true-CASP variant of {\seq} in the language of
{\clingcon}. We further adapted that encoding to the {\ez} language by
means of simple syntactic transformations. Here we provide a 
review of details of the {\seq} formalizations using pure-ASP and the \ez language that
we find most relevant to this presentation. The reader can refer to~\ref{sec:ez-language} for details on the syntax used. 
The non-domain predicates of the pure-ASP encoding are 
$\mathit{leafPos}$, $posColor$, $posCost$. Intuitively,
 $\mathit{leafPos}$ is responsible for assigning a position to a
leaf,  $posColor$ is responsible for assigning a color to each position,
 $posCost$ carries information on costs associated with
each leaf. Some rules used to define these relations are given in Figure~\ref{fg:wseqrules}.
\begin{figure}[htbp]
\bsmall
\fbox{
$
\begin{array}{l}
\mbox{\% Give each leaf a location in the sequence} \\
1\{ leafPos(L,N) : location(N) \}1 \hif leaf(L). \\
\mbox{\% No sharing of locations} \\
\hif leafPos(L1, N), leafPos(L2, N), location(N), L1 \neq L2. \\
\\
\mbox{\% green if (weight(right) + card(right)) $<$ (weight(left) + leafCost(right))}\\
posColor(1,green) \hif leafPos(L1,0), leafPos(L2,1), \\
\hspace*{1.48in} leafWeightCardinality(L1,WL,CL), \\
\hspace*{1.48in} leafWeightCardinality(L2,WR,CR), \\
\hspace*{1.48in} leafCost(L2,W3),\\
\hspace*{1.48in} W1=WR+CR,\ W2=WL+W3,\\
\hspace*{1.48in} W1 < W2. \\
\\
\mbox{\% posCost for first coloredPos}\\
posCost(1,W) \hif posColor(1,green), leafPos(L,1), \\
\hspace*{1.2in} leafWeightCardinality(L,WR,CR), \\
\hspace*{1.2in} max\_total\_weight(MAX), \\
\hspace*{1.2in} W=WR+CR, W \leq MAX. \\
posCost(1,W) \hif \lpnot posColor(1,green), leafPos(L1,0), leafPos(L2,1),  \\
\hspace*{1.2in} leafWeightCardinality(L1,WL,CL), leafCost(L2,WR),\\
\hspace*{1.2in} max\_total\_weight(MAX),\\
\hspace*{1.2in} W=WL+WR,\ W \leq MAX. \\
\\
\mbox{\% Acceptable solutions}\\
acceptable \hif \#sum[nWeight(P,W)=W:coloredPos(P)] MAX, \\
\hspace*{.92in} max\_total\_weight(MAX).\\
\hif \lpnot acceptable.
\end{array}
$
}
\esmall
\caption{Some typical rules of the pure-ASP language formalization of \wseq.}\label{fg:wseqrules}
\end{figure}

The first two rules in Figure~\ref{fg:wseqrules} assign a distinct
location to each leaf. The next rule is 
part of the color assignment. The following two rules are part of the cost determination.
The final two rules ensure that the total cost is within the specified limit.

The main difference between the
pure-ASP and true-CASP encodings is in the treatment of the cost values of
the leaves. We first note that  cost  predicate
 $posCost$  in the pure-ASP encoding
 is ``functional". In other words, 
when this predicate occurs in an answer
set, its first argument uniquely determines its second argument.
 Often, such functional
predicates in ASP encodings can be replaced by ez-atoms\footnote{We
  abuse the term ez-atom and refer to ``non-ground'' atoms of the \ez
  language that result in ez-atoms by the same name.} in 
CASP encodings.
Indeed, this is the case in the weighted-sequence
problem. Thus in the true-CASP encoding, the definition of $posCost$
is replaced by suitable ez-atoms, making it possible to evaluate
cost values by CSP techniques. This approach is expected to
benefit performance especially when the cost values are large.
Some of the corresponding rules follow:
\bsmall
\[
\begin{array}{l}
\mbox{\% posCost for first coloredPos}\\
required(posCost(1)=W)\hif posColor(1,green), leafPos(L,1), \\
\hspace*{1.1in} leafWeightCardinality(L,WR,CR), W=WR+CR.\\
required(posCost(1)=W) \hif \lpnot posColor(1,green), \\
\hspace*{1.1in} leafPos(L1,0), leafPos(L2,1), \\
\hspace*{1.1in} leafWeightCardinality(L1,WL,CL), leafCost(L2,WR), \\
\hspace*{1.1in} W=WL+WR. \\
\\
\mbox{\% Acceptable solutions}\\
required(sum([posCost/1],\leq,MV)) \hif max\_total\_weight(MV).
\end{array}
\]
\esmall
The first two rules are rather straightforward translations of the ASP equivalents.
The last rule uses a global constraint to ensure acceptability of the total cost.

The pure-CSP encoding is obtained from the true-CASP encoding by
replacing the definitions of $\mathit{leafPos}$ and $posColor$ predicates by constraint
atoms.
The replacement is based on the observation that $\mathit{leafPos}$ and $posColor$
are functional.
\bsmall
\[
\begin{array}{l}
\mbox{\% green if (weight(right) + card(right)) $<$ (weight(left) + leafCost(right))}\\
is\_green(1,L1,L2) \hif leafWeightCardinality(L1,WL,CL), \\
\hspace*{1.48in} leafWeightCardinality(L2,WR,CR), \\
\hspace*{1.48in} leafCost(L2,W3), \\
\hspace*{1.48in} W1=WR+CR,\ W2=WL+W3,\\
\hspace*{1.48in} W1 < W2.\\

required(posColor(1)=green \leftarrow (leafPos(L1)=0 \ \land \ leafPos(L2)=1)) \hif\\
\aspindent leaf(L1),\ leaf(L2),\ is\_green(1,L1,L2).
\end{array}
\]
\esmall
As shown by the last rule, color assignment requires the use of reified constraints. It is important to note that symbol $\leftarrow$ within the scope of $required$  stands for material implication.
Color names are mapped to integers by introducing additional variables. For
example, variable $green$ is associated with value $1$ by a variable declaration
$cspvar(green,1,1)$.
Interestingly, no ez-atoms are needed for the definition of $leafPos$. The role of
the choice rule above is implicitly played by the variable declaration
$$cspvar(leafPos(L),0,N-1) \hif leaf(L), location(N).$$


\smallskip
\noindent
The {\bf incremental scheduling} (\is) domain stems from a problem
occurring in commercial
printing. In this domain,  a schedule is maintained up-to-date with respect to 
 jobs being added and  equipment going offline.
A problem description includes a set of
devices, each with predefined number of instances (slots for jobs), and a set of 
jobs to be produced. The penalty for a 
job being late is computed as $td \cdot imp$, where $td$ 
is the job's tardiness and $imp$ is a positive integer denoting 
the job's importance.
The total penalty of a schedule is the sum of the penalties 
of the jobs.
The task is to find a schedule whose
total penalty is no larger than the value  specified in
 a problem instance.
We direct the reader to~\cite{bal11} for more details on this domain.
We start by describing the pure-CSP encoding and then illustrate how it relates to the true-CASP encoding.

The pure-CSP encoding used in our experiments is the
official competition encoding submitted
to {\aspcomp} by the {\ezcsp} team.
In that encoding, constraint atoms are used for (i) 
assigning start times to jobs, (ii) selecting which
device instance will perform a job, and (iii) 
calculating tardiness and penalties.
Core rules of the encoding are shown in Figure~\ref{fig:ifex}.
\begin{figure}[htbp]
\begin{center}
\fbox{
$
\begin{array}{l}
\mbox{\% Assignment of start times: cumulative constraint}\\
required(cumulative([st(D)/2], \\
\hspace*{.65in}[operation\_len\_by\_dev(D)/3], \\
\hspace*{.65in}[operation\_res\_by\_dev(D)/3], \\
\hspace*{.65in}N)) \hif \\
\aspindent instances(D,N). \\
\mbox{\% Instance assignment}\\
required((on\_instance(J1) \neq on\_instance(J2)) \ \lor  \\
\hspace*{.65in}(st(D,J2) >= st(D,J1) + Len1) \ \lor \\
\hspace*{.65in}(st(D,J1) >= st(D,J2) + Len2)) \ \hif \\
\aspindent instances(D,N), N > 1, \\
\aspindent job\_device(J1,D), job\_device(J2,D), J1 \neq J2, \\
\aspindent job\_len(J1,Len1), job\_len(J2,Len2). \\
\mbox{\% Total Penalty}\\
required(sum([penalty/1],=,tot\_penalty)). \\
required(tot\_penalty \leq K) \hif max\_total\_penalty(K).
\end{array}
$
}
\end{center}
\caption{Rules of the pure-CSP formalization of \is.}\label{fig:ifex}
\end{figure}

The ez-atom of the first rule uses a global constraint to specify that the start times must be 
assigned in such a way as to ensure that no more than $n_d$ jobs are executed at any time,
where $n_d$ is the number of instances of a given device $d$.
The  ez-atom of the second rule uses reified constraints with the $\lor$ connective (disjunction) to guarantee that at most one job
is executed on a device instance at every time.
The  ez-atom of the third rule uses a global constraint to define total penalty. The last rule
restricts  total penalty to be within the allowed maximum value.

\st
The true-CASP encoding was obtained from the pure-CSP encoding
by introducing a new relation $on\_instance(j,i)$, stating
that job $j$  runs on device-instance $i$. 
The rules formalizing the assignment of device instances in the
pure-CSP encoding were replaced by ez-atoms. For example, the second
rule from Figure~\ref{fig:ifex} was replaced by: 
\bsmall
\[
\begin{array}{l}
1 \{ on\_instance(J,I) : instance\_of(D,I) \} 1 \hif job\_device(J,D). \\
\end{array}
\]
\[
\begin{array}{l}
required((st(D,J2) \geq st(D,J1) + Len1)\ \lor \\
\hspace*{0.65in}(st(D,J1) \geq st(D,J2) + Len2)) \hif \\
\aspindent on\_instance(J1,I), on\_instance(J2,I), \\
\aspindent instances(D,N), N > 1, \\
\aspindent job\_device(J1,D), job\_device(J2,D), J1 \neq J2, \\
\aspindent job\_len(J1,Len1), job\_len(J2,Len2). \\
\end{array}
\]
\esmall
The main difference with respect to the ez-atom of the pure-CSP encoding is
the introduction of a choice rule to select an instance $I$ for a job $J$.
The constraint that each instance processes at most one job at a time
is still encoded using an ez-atom.

\st
Finally, the pure-ASP encoding was obtained from the true-CASP encoding
by introducing suitable new relations, such as $start(j,s)$
and $penalty(j,p)$, to replace all remaining ez-atoms. 
The rules that replace the first rule in Figure~\ref{fig:ifex}  follow:
\bsmall
\[
\begin{array}{l}
1 \{ start(J,S) : time(S) \} 1 \hif job(J). \\
\\
\hif on\_instance(J1,I), on\_instance(J2,I), J1 \neq J2, \\
\aspindhif job\_device(J1,D), job\_device(J2,D), \\
\aspindhif start(J1,S1), job\_len(J1,L1), start(J2,S2), \\
\aspindhif S1 \leq S2, S2 < S1 + L1.
\end{array}
\]
\esmall
The last two rules in Figure~\ref{fig:ifex} are replaced by the rules in the pure-ASP encoding:
\bsmall
\[
\begin{array}{l}
tot\_penalty(TP) \hif ~TP~[~ penalty(J,P)=P ~]~TP. \\
\hif~ \lpnot [ penalty(J,P) = P ] Max, max\_total\_penalty(Max).\\
\end{array}
\]
\esmall
In the {\bf reverse folding} (\rf) domain,
one manipulates a sequence of $n$
pairwise connected segments located on a 2D plane in order to take the sequence from
an initial configuration to a goal configuration.
The sequence is manipulated by pivot moves: rotations of a segment around its
starting point by 90 degree in either direction.
A pivot move on a segment causes
the segments that follow to rotate around the same
center. 
Concurrent pivot moves are prohibited. At the end of each
move, the segments in the sequence must not intersect.
A problem instance specifies the number of segments, 
the goal configuration, and required number of moves denoted by $t$.
The task is to find a sequence of exactly
$t$ pivot moves that produces the goal configuration.

The true-CASP encoding used for our experiments is from
the official {\aspcomp2011} submission package of the {\ezcsp} team.
In this encoding, relation $pivot(s,i,d)$ states that at step
$s$ the $i^{th}$ segment is rotated in direction $d$.
The effects of pivot moves are described by 
ez-atoms, which allows us to carry out the corresponding calculations
with CSP techniques.
\bsmall
\[
\begin{array}{l}
pivot(1,I,D) \hif first(I), requiredMove(I,D).\\
pivot(N1,I1,D1) \hif pivot(N2,I2,D2), N1 = N2 + 1, \\
\hspace*{1.39in} requiredMove(I1,D1), requiredMove(I2,D2),\\
\hspace*{1.39in} next(I1,I2).\\
\end{array}
\]
\[
\begin{array}{l}
\mbox{\% Effect of pivot(t,i,d)}\\
required(tfoldy(S2,I)=tfoldx(S1,P)-tfoldx(S1,I)+tfoldy(S1,P)) \hif \\
\aspindent step(S1), step(S2), S2=S1+1, \\
\aspindent pivot(S1,P,clock), \\
\aspindent index(I), I \geq P. \\
required(tfoldy(S2,I)=tfoldx(S1,I)-tfoldx(S1,P)+tfoldy(S1,P)) \hif \\
\aspindent step(S1), step(S2), S2=S1+1, \\
\aspindent pivot(S1,P,anticlock), \\
\aspindent index(I), I \geq P.
\end{array}
\]
\esmall
The first two rules are some of the rules used for determining
the pivot rotations. The determination is based on the technique
described in \cite{bl12}. The last two rules are part of the calculation
of the effects of pivot moves. Note that $tfoldx(s,i)$ and $tfoldy(s,i)$ denote the $x$ and $y$ coordinates of the start of segment $i$ at step $s$.

The pure-ASP encoding was obtained from the true-CASP encoding
by adopting an ASP-based formalization of the effects of
pivot moves. This was accomplished by introducing two new
relations, ${tfoldx}(s,i,x)$ and ${tfoldy}(s,i,y)$,
stating that the new start of segment $i$ at step $s$
is $\tbeg x, y \tend$.
The definition of the relations is provided by
suitable ASP rules, such as:
\bsmall
\[
\begin{array}{l}
tfoldy(S+1,I,Y2) \hif tfoldx(S,I,X1), pivot(S,P,D), I \geq P, \\
\hspace*{1.5in} tfoldx(S,P,XP), tfoldy(S,P,YP), X0 = X1 - XP,\\
\hspace*{1.5in} rotatedx(D,X0,Y0), Y2 = Y0 + YP. \\
rotatedx(clock,X,-X) \hif xcoord(X). \\
xcoord(-2*N..2*N) \hif length(N).
\end{array}
\]
\esmall
Differently from the previous domains, for {\rf} we were unable to formulate a pure-CSP variant of the true-CASP encoding. Thus, we resorted to the encoding described in \cite{dfp11}. This encoding leverages a mapping from action language $\mathcal{B}$ \cite{gel98} statements to numerical constraints, which are then solved by a CLP system.

\section{Experimental Results}\label{sec:experiments}

\begin{figure}[t]
\begin{center}
\includegraphics[clip=true,trim=40 250 40 250,width=1\columnwidth]{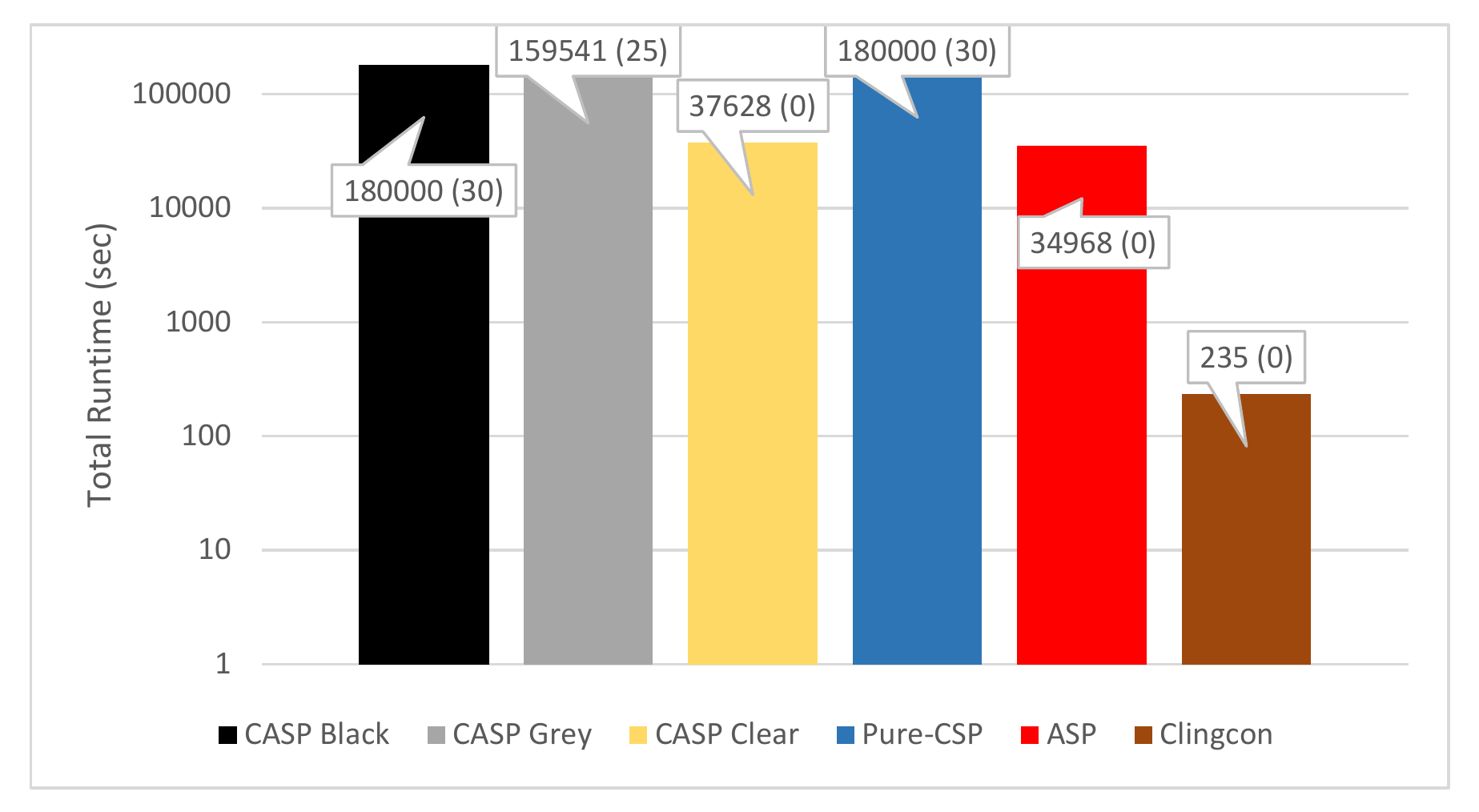}
\end{center}
\caption{Performance on {\wseq} domain: total times in logarithmic scale}
\label{fig:wseq-totals}
\end{figure}
\begin{figure}[t]
\begin{center}
\includegraphics[clip=true,trim=40 110 40 110,width=1\columnwidth]{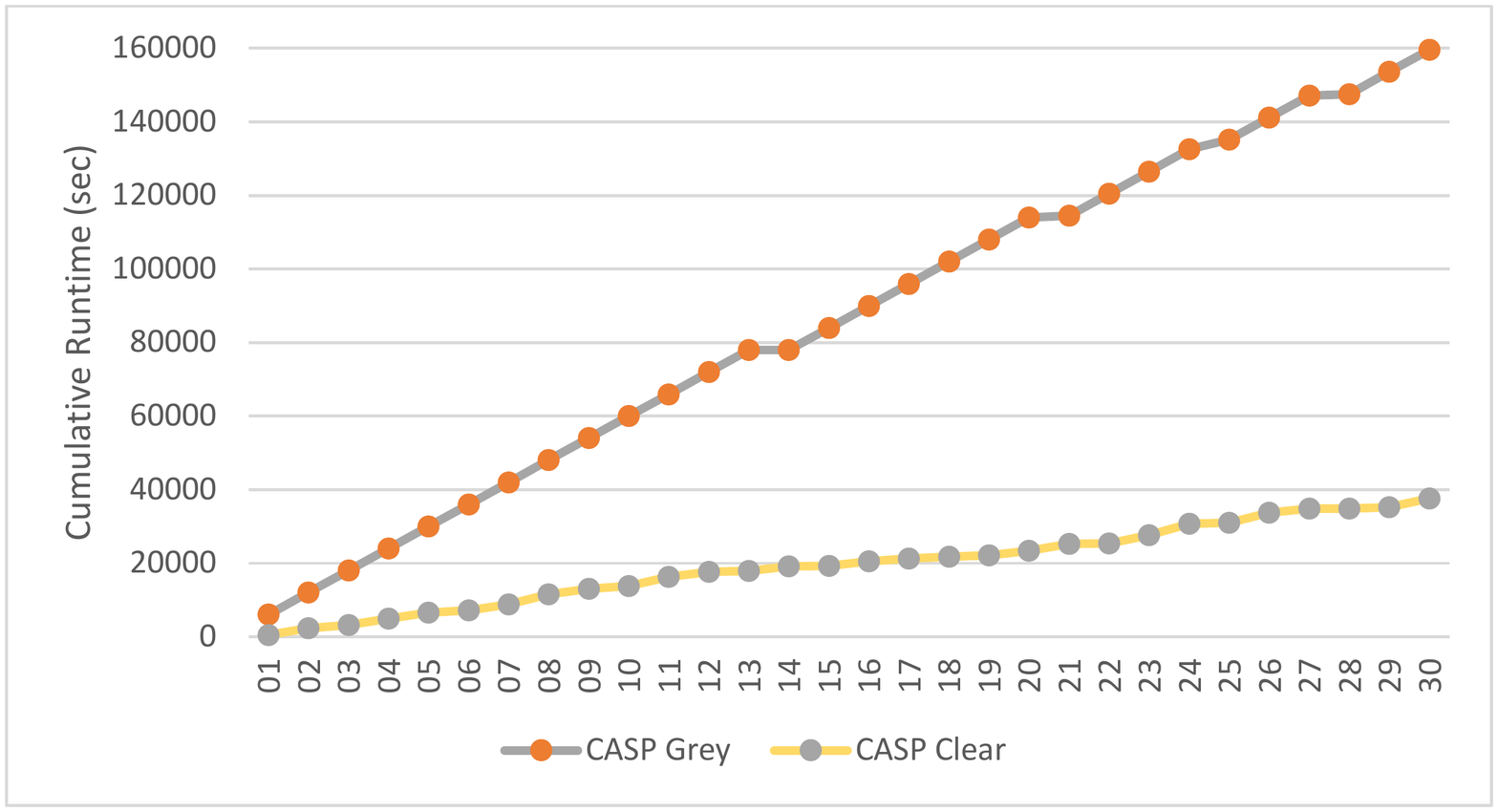}
\end{center}
\caption{Performance on {\wseq} domain: cumulative view (\gbox and \cbox)}
\label{fig:wseq-cumulative}
\end{figure}

\st
The experimental comparison of the integration schemas
was conducted on a computer with an Intel Core i7
processor at 3GHz and running Fedora Core 16.
The memory limit for each process 
and the timeout were set to $1$ GB
RAM\footnote{The instances that resulted in an out-of-memory were also tested with 4 GB RAM, with no change in the outcome.} and $6,000$ seconds respectively. A single processor core was used for every experiment.

The version of {\ezcsp}
used in the experiments was 1.6.20b49.
This version implements the {\bbox}, {\gbox}, and {\cbox}  integration schemas, when
 suitable API interfaces are available in the base
solver. One answer set solver that provides such interfaces is
{\cmodels}, which for this reason was chosen as base solver for the
experiments.  
It is worth noting that the development of the API in
{\cmodels} was greatly facilitated by the API provided by
{\minisat} v. 1.12b supporting non-clausal
constraints~\cite{minisat-manual} ({\minisat} forms the main
inference mechanism of {\cmodels}).
In the experiments, we used {\cmodels} version
3.83 as the base solver
and  {\bprolog} 7.4 as the theory solver.\footnote{We note that {\bprolog} is the default theory solver of {\ezcsp}. Command-line option {\tt --solver cmodels-3.83} instructs \ezcsp to invoke \cmodels 3.83 using the {\bbox} integration schema.
Command-line options {\tt --cmodels-incremental} and {\tt --cmodels-feedback}
instruct {\ezcsp} to use, respectively, the {\gbox} and {\cbox} integration schema. In these two cases, {\cmodels} 3.83 is automatically selected as the base solver.}
Answer set solver {\cmodels} 3.83 (with the inference mechanism of {\minisat} v. 1.12b)
was also used for the
experiments with
 the pure-ASP encodings. Unless otherwise specified, for all solvers we used their default configurations.

The executables used in the experiments and the
encodings can be downloaded, respectively, from 
\begin{itemize}
\item
\bsmall
\url{http://www.mbal.tk/ezcsp/int_schemas/ezcsp-binaries.tgz},\\
and
\esmall
\item
\bsmall
\url{http://www.mbal.tk/ezcsp/int_schemas/experiments.tgz}.
\esmall
\end{itemize}

In order to provide a frame of reference with respect
to the state of the art in CASP, the results also include 
performance information for {\clingcon} 2.0.3
on the true-CASP encodings adapted to the language of {\clingcon}.
\tcb{We conjecture that the choice of constraint solver by {\clingcon} (namely, {\gecode}) together with theory propagation
is the reason  for {\clingcon}'s better performance in a number of the experiments. Yet, in the context of our experiments,  the performance of {\clingcon} w.r.t. {\ezcsp} is
irrelevant. Our work is a comparative study of the impact of the
different integration 
schemas for a fixed selection of a base and theory solver pair. System {\ezcsp} provides us with essential means to perform this study.
}

In all figures presented:
CASP Black, CASP Grey, CASP Clear denote 
{\ezcsp} implementing respectively {\bbox}, {\gbox} and {\cbox}, and
running a true-CASP encoding; 
Pure-CSP denotes {\ezcsp} implementing {\bbox} running
a pure-CSP encoding (note that for pure-CSP encodings there is no
difference in performance between the integration schemas); 
ASP denotes {\cmodels} running a pure-ASP encoding;
Clingcon  denotes
{\clingcon} running a true-CASP encoding. Each configuration is
associated with the same color in all figures. A pattern is applied to the filling of the bars whenever the bar goes off-chart. The numbers in the
overlaid boxes report the time in seconds and, in parentheses, the
total number of timeouts and out-of-memory. 

\st
We begin our analysis with {\wseq} (Figures
\ref{fig:wseq-totals} and \ref{fig:wseq-cumulative}).
The total times across all
the instances for all solvers/encodings pairs considered are shown in
Figure~\ref{fig:wseq-totals}. Because of the large difference between best and worst performance, a logarithmic scale is used. For uniformity of presentation, in the
charts 
out-of-memory conditions and timeouts are
both rendered as out-of-time results. The instances used in the experiments are the 30
instances available via {\aspcomp}.
Interestingly, answer set solver {\cmodels} on the pure-ASP encoding has excellent performance, comparable to the best performance obtained with CASP encodings by \ezcsp.
Of the CASP encodings, the true-CASP encoding running in {\bbox} times out on every instance.
\st 
Figure~\ref{fig:wseq-cumulative} thus focuses on the cumulative run times of
{\cbox} and {\gbox} (on the true-CASP encoding). The numbers on the
horizontal axis identify the instances, while the vertical axis is for
the cumulative run time, that is, the value for instance $n$ is the
sum of the run times for instances $1 \ldots n$. Cumulative times were
chosen for the per-instance figures because they make for a more
readable chart when there is large variation between the run times for
the individual instances. As shown in
Figure~\ref{fig:wseq-cumulative},  
the true-CASP encoding running in {\cbox} performs \emph{substantially}
better than {\gbox}. 
 This demonstrates that,
for this domain, the tight integration schema has an advantage.

\begin{figure}[h]
\begin{center}
\includegraphics[clip=true,trim=40 250 40 250,width=1\columnwidth]{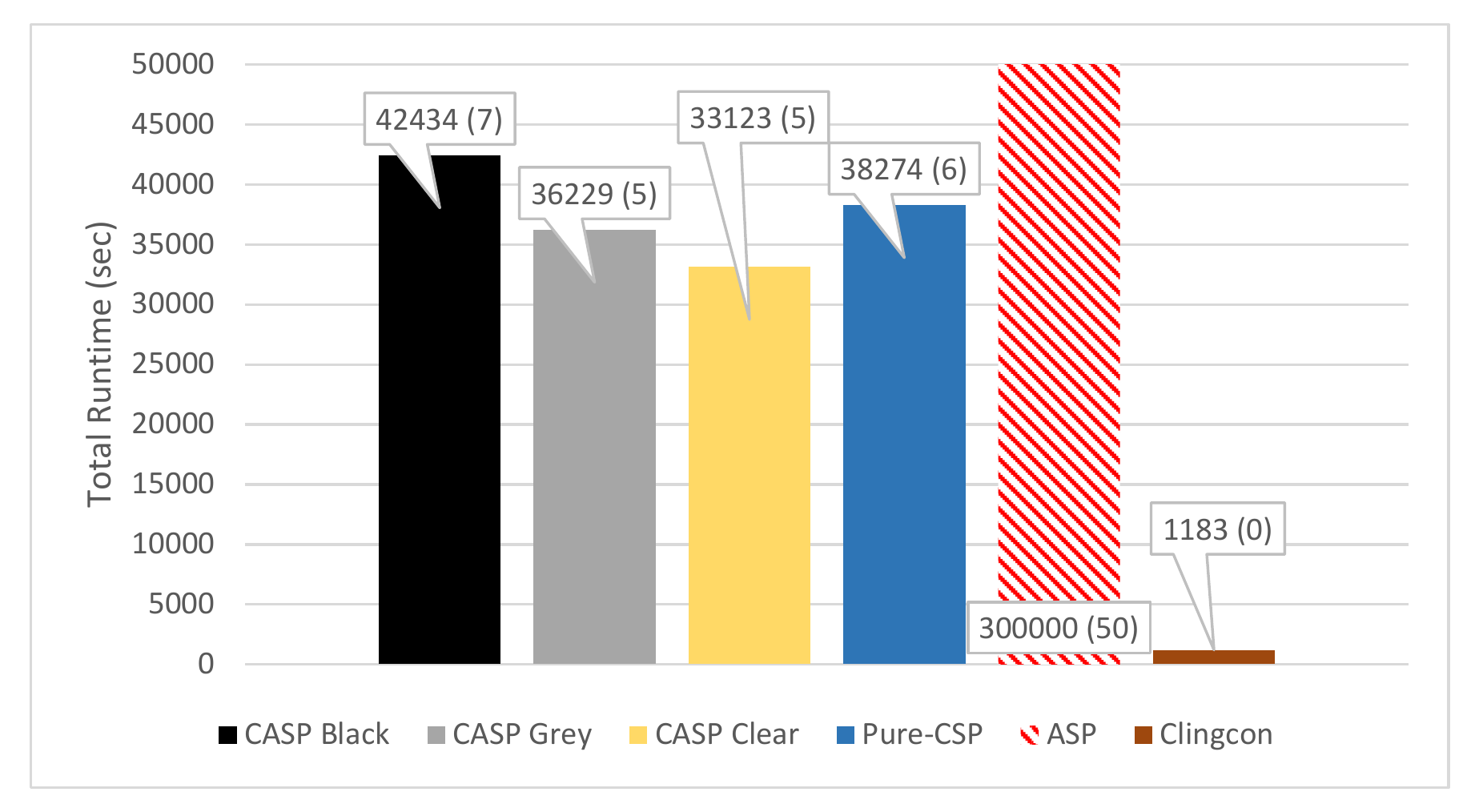}
\end{center}
\caption{Performance on {\is} domain, easy instances: total times (ASP encoding off-chart)}
\label{fig:is-easy-totals}
\end{figure}
\begin{figure}[h]
\begin{center}
\includegraphics[clip=true,trim=40 110 40 110,width=1\columnwidth]{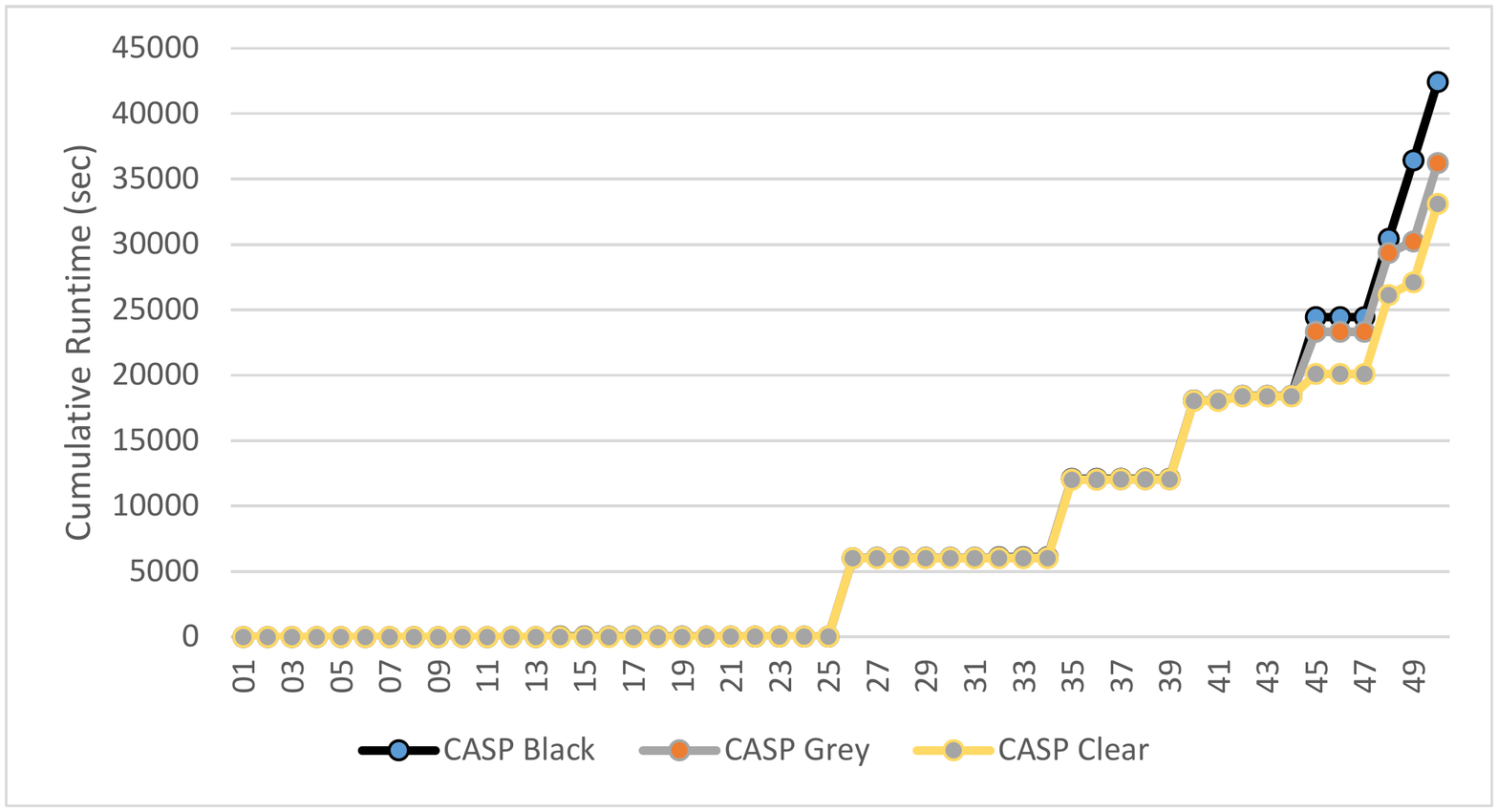}
\end{center}
\caption{Performance on {\is} domain, easy instances: cumulative view}
\label{fig:is-easy-cumulative}
\end{figure}

\begin{figure}[h]
\begin{center}
\includegraphics[clip=true,trim=40 250 40 250,width=1\columnwidth]{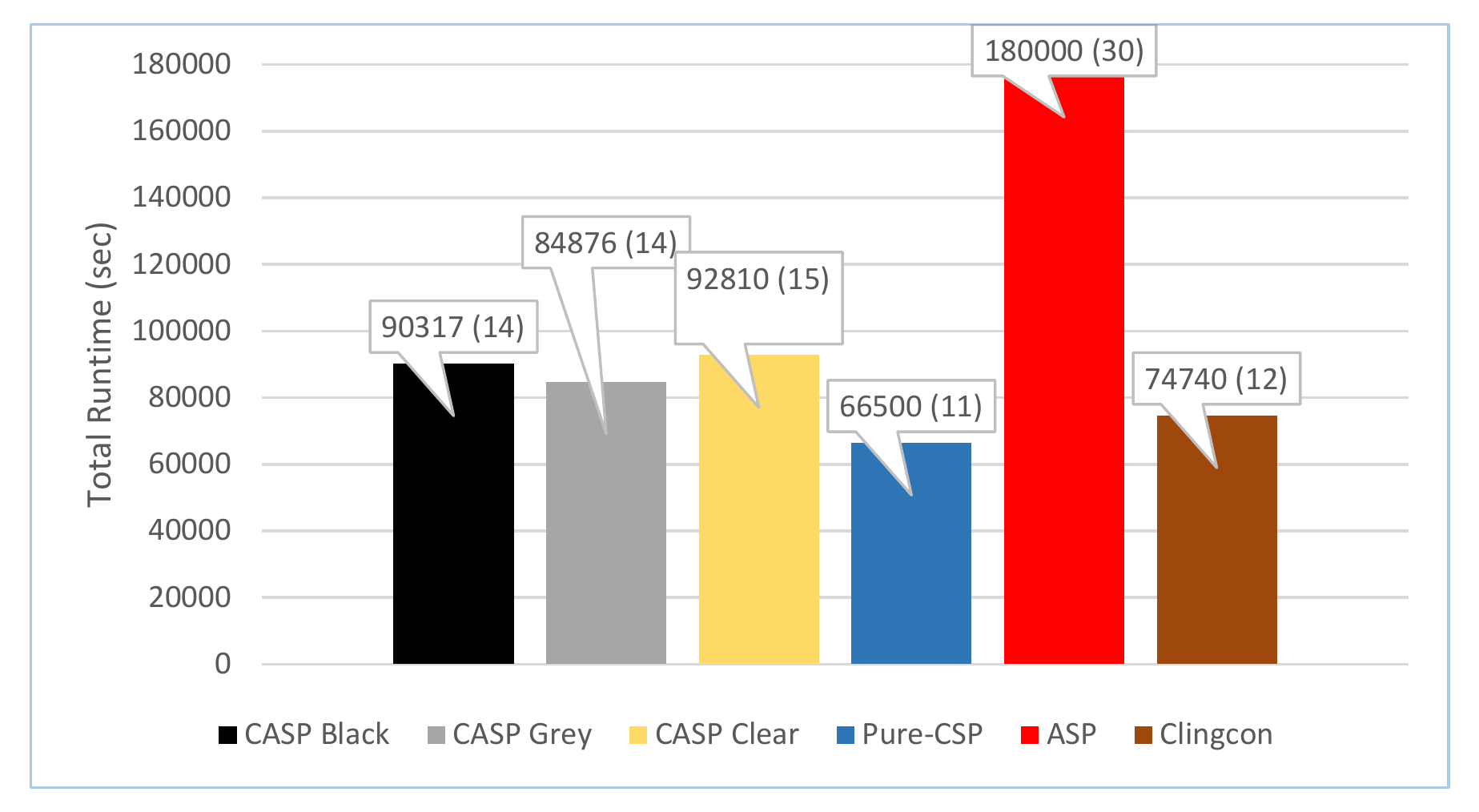}
\end{center}
\caption{Performance on {\is} domain, hard instances: overall view}
\label{fig:is-hard-totals}
\end{figure}
\begin{figure}[h]
\begin{center}
\includegraphics[clip=true,trim=40 250 40 250,width=1\columnwidth]{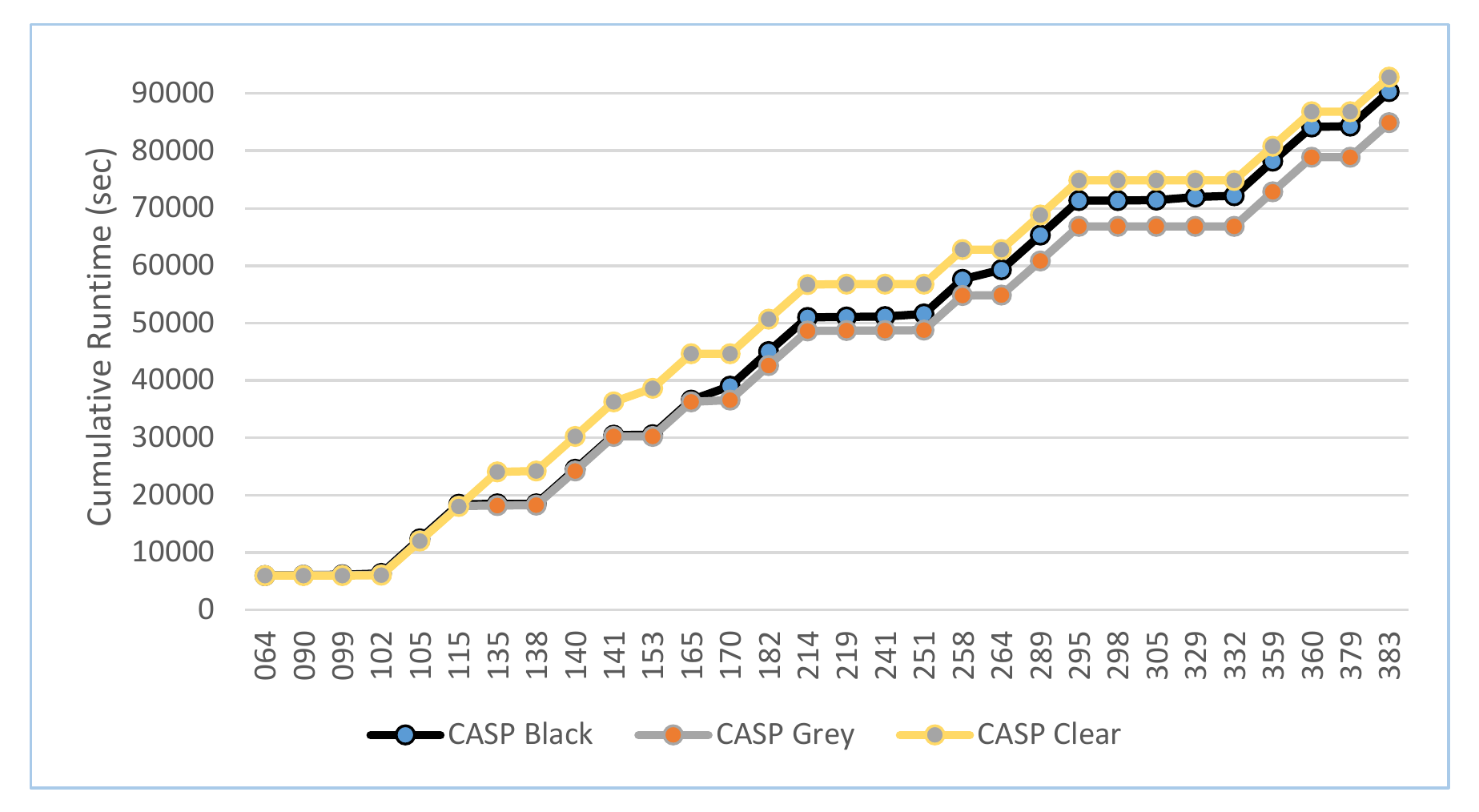}
\end{center}
\caption{Performance on {\is} domain, hard instances: cumulative view}
\label{fig:is-hard-cumulative}
\end{figure}

\st
In case of the {\is} domain we considered two sets of experiments. In
the first one (Figures~\ref{fig:is-easy-totals} and~\ref{fig:is-easy-cumulative}), we
used the 50 official instances from {\aspcomp}. We refer to these
instances as {\em easy}, since the corresponding run times are rather small.
Figure~\ref{fig:is-easy-totals}
provides a comparison of the total times.
Judging by the total times, tight integration schemas
appear to have an advantage, allowing the true-CASP encoding to outperform
the pure-CSP encoding. As one might expect, the
best performance for the true-CASP encoding is from the {\cbox} integration schema.
In this case the early pruning of the search space 
made possible by the {\cbox} architecture seems to yield substantial benefits.
As expected, {\gbox} is also faster than {\bbox}, while {\cmodels} on
the pure-ASP encoding runs out of memory in all the instances.

\st
%

\st
The second set of experiments for the {\is} domain (Figures \ref{fig:is-hard-totals} and \ref{fig:is-hard-cumulative})
consists of~$30$ instances that we generated to be substantially more complex than the ones
from {\aspcomp}, and that are thus called {\em hard}. As discussed below, this second set of experiments
reveals a remarkable change in the behavior of
solver/encodings pairs
when the instances require more computational effort.
The process we followed to generate the~$30$ {\em hard} instances
consisted in (1) generating randomly~$500$ fresh instances;
(2) running the true-CASP
encoding with the {\gbox} integration schema on them
with a timeout of~$300$ seconds; (3) 
selecting randomly, from those,~$15$ instances that resulted
in timeout and~$15$ instances that were solved in~$25$ seconds
or more.
The numerical parameters used in the process
were selected with the purpose of identifying more challenging instances
than those from the {\em easy} set and were based on the results on that set.
The execution times reported in
Figure~\ref{fig:is-hard-totals} 
clearly indicate the
level of difficulty of the selected instances (once again, {\cmodels} runs out of memory).
Remarkably, these more difficult instances
are solved more efficiently by the pure-CSP encoding that relies
only on the CSP solver.
In fact, the pure-CSP encoding outperforms every other method of computation
(\emph{including} {\clingcon} on true-CASP encoding).
More specifically, solving the instances with the true-CASP 
encoding takes between $30\%$ and
$50\%$ longer than with the pure-CSP encoding.
This was not the isolated effect of a few instances, but  rather a
constant pattern throughout the experiment. A possible explanation for
this phenomenon is that domain {\is} is overall best suited 
to the CSP solving procedures. It seems natural for the difference in performance
to become more evident as the problem instances become more challenging, when other
factors such as overhead play less of a role. This conjecture is compatible with 
the difference in performance observed earlier on the easy instances.

Another remarkable aspect highlighted by
Figure~\ref{fig:is-hard-cumulative} is that {\cbox} is 
outperformed by {\gbox}. This is the opposite of what was observed on
the easy instances and highlights the fact that there is no  single-best
integration schema, even when one focuses on true-CASP encodings. We
hypothesize this to be due to the nature of the underlying scheduling
problem, which is hard to solve, but whose relaxations (obtained by
dropping one or more constraints) are relatively easy. Under these
conditions, the calls executed by {\cbox} to the theory solver are
ineffective at pruning the search space and incur a non-negligible
overhead. (The performance of {\clingcon} is likely affected by the
same behavior.) In {\gbox}, on the other hand, no time is wasted
trying to prune the search space of the base solver, and all the time
spent in the theory solver is dedicated to solving the final CSP. The
performance of {\bbox} is likely due to the minor efficiency of its
integration schema compared to {\gbox}.  

\begin{figure}[h]
\begin{center}
\includegraphics[clip=true,trim=40 250 40 250,width=1\columnwidth]{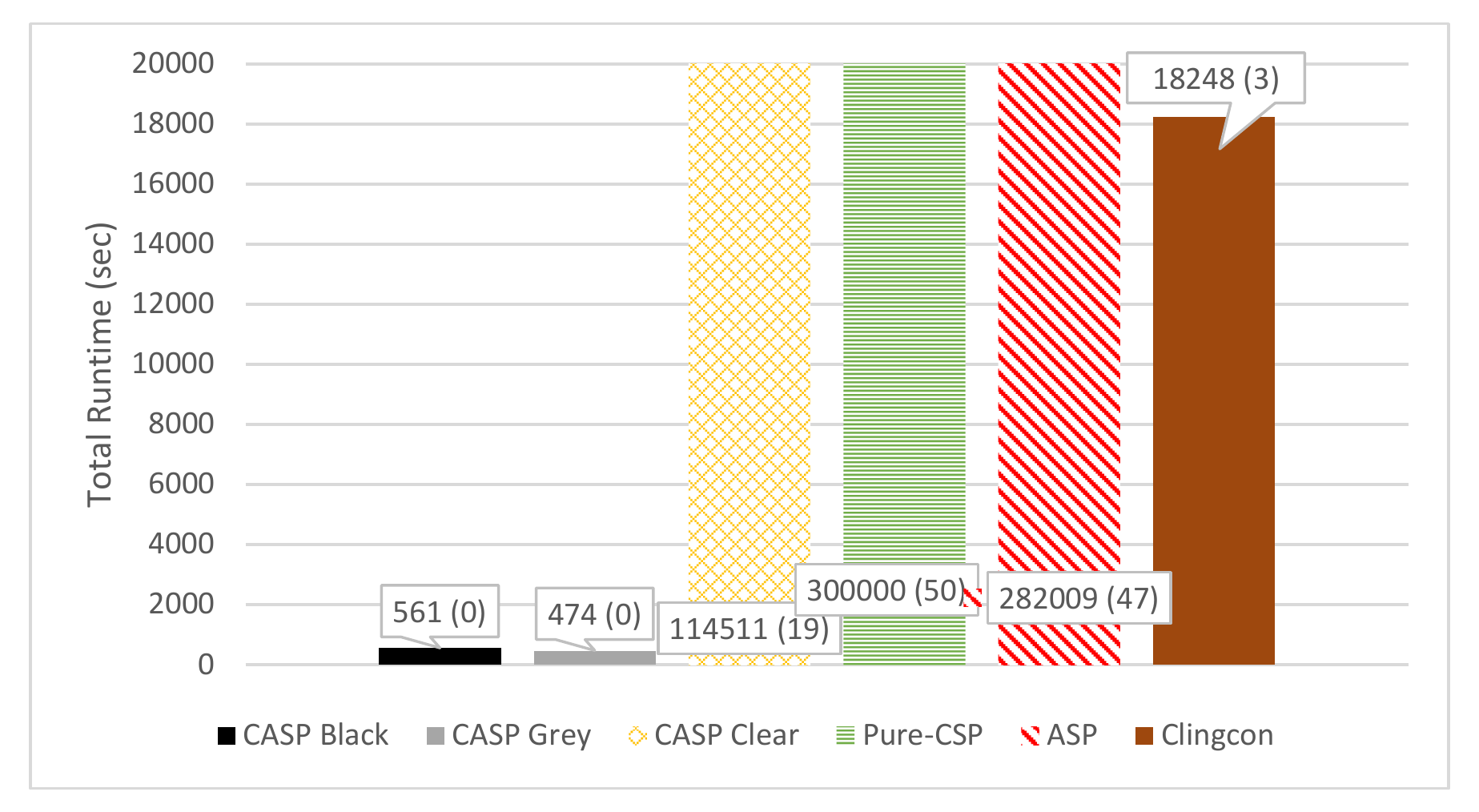}
\end{center}
\caption{Performance on {\rf} domain: total times (detail of 0-20,000sec execution time range, {\cbox} and pure-ASP off-chart)}
\label{fig:rf-totals}
\end{figure}
\begin{figure}[h]
\begin{center}
\includegraphics[clip=true,trim=40 100 40 110,width=1\columnwidth]{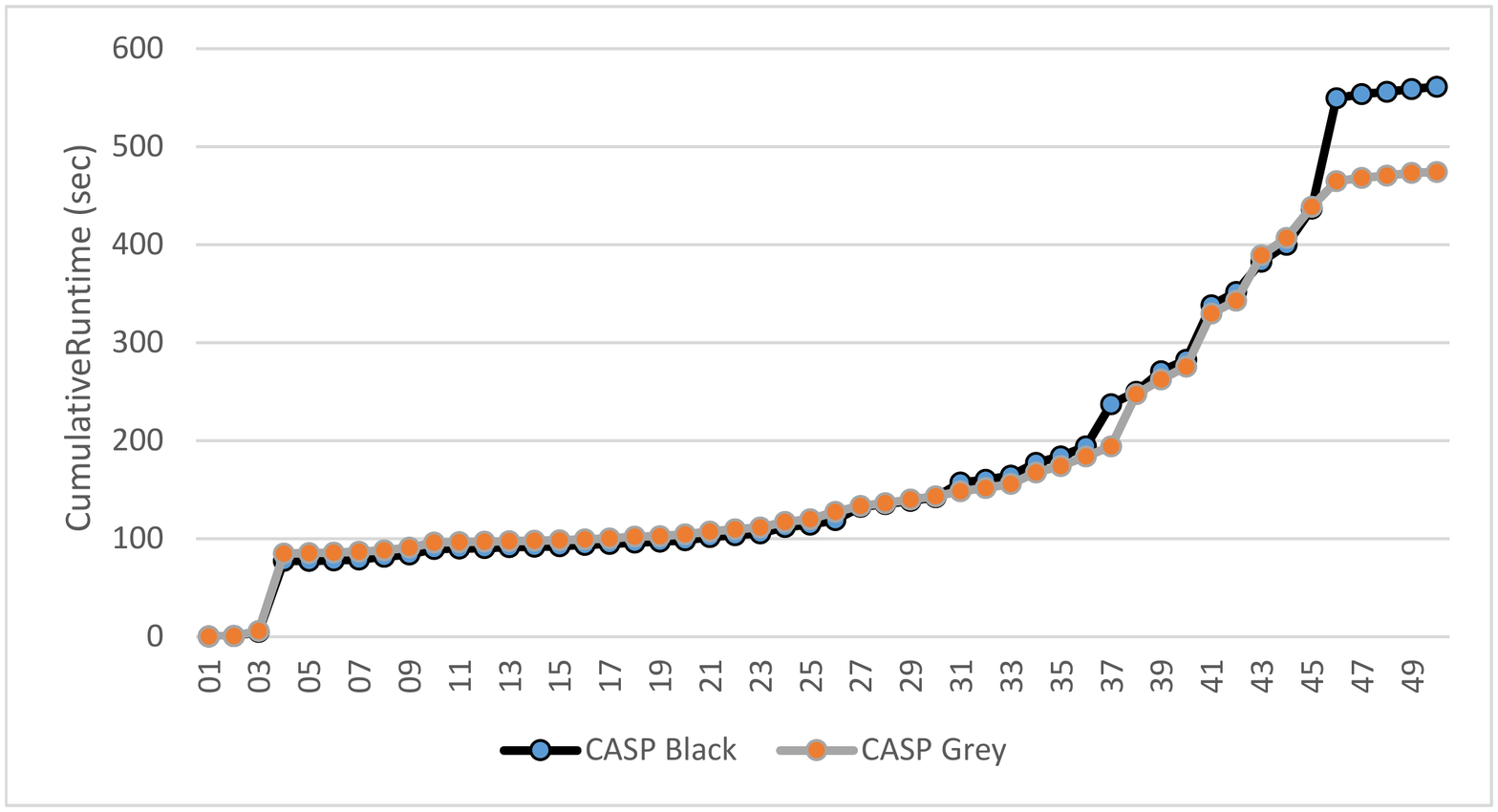}
\end{center}
\caption{Performance on {\rf} domain: cumulative view (\bbox and \gbox)}
\label{fig:rf-cumulative}
\end{figure}

\st
The final experiment focuses on the {\rf} domain (Figures
\ref{fig:rf-totals} and \ref{fig:rf-cumulative}). 
The instances used
in this experiment are the 50 official instances from {\aspcomp}.
The total execution times are presented in 
Figure~\ref{fig:rf-totals}. 
Although the instances for this domain are comparatively easy, as
suggested by the {\bbox} and {\gbox} times, some of the configurations
have high total execution times.
The {\cbox} encoding is also off-chart, due to timeouts on~$19$ instances.
This is a substantial difference in performance compared to the other
true-CASP configurations, upon which we expand later in this
section. Surprisingly, the total time of {\clingcon} is also close to
off-chart. 
Upon closer inspection, we have found this to be due to~$3$
instances for which {\clingcon} runs out of memory. 
This is an interesting instance
of the trade-off between speed of execution and performance stability,
considering that on the other instances {\clingcon} is very fast. 
The per-instance execution times for {\gbox} and {\bbox} are detailed
in Figure~\ref{fig:rf-cumulative}. 
The figure highlights the very similar performance of the two schemas,
with {\bbox} losing only in the final 10\% of the instances in spite
of its higher overhead. This is likely due to the simplicity of the
{\rf} problem: most extended answer sets can be found with little
backtracking between base and theory solver, and thus the difference
between the two schemas has little bearing on the execution
times. Similarly to the hard instances of the {\is} domain, the better
performance of {\bbox} and {\gbox} in comparison to {\cbox} 
can be explained by the fact that,  
in this domain, frequent checks
with the theory solver add overhead but are overall ineffective at pruning
the search. 



\section{A Brief Account on Related Systems}\label{sec:relsys}

In the introduction we mentioned solvers
{\acsolver}~\cite{mel08},
{\clingcon}~\cite{geb09}, 
{\sc idp}~\cite{idp}, {\sc inca}~\cite{dre11a}, {\sc
        dingo}~\cite{jan11},  {\sc mingo}~\cite{liu12}, \aspmt~\cite{Bartholomew2014}, and~\ezsmt~\cite{sus16a}. In this section we briefly remark on this variety of CASP systems. This is not intended as a detailed comparison  between the systems, but as a quick summary.

At a high-level abstraction, one may easily relate the architectures of the {\clingcon}, {\acsolver},  {\sc idp}, and {\sc inca} to that of \ezcsp.
Given a CASP program, all of these systems first utilize an answer set solver to compute a part of an answer set for an asp-abstraction and then utilize a constraint programming system to solve a resulting csp-abstraction. All of these systems implement the \cbox integration schema. 
Table \ref{tab:solvers} provides a summary of base solvers and theory solvers utilized by them.

\begin{table}
\begin{tabular}{||l|l|l||}
        &Base Solver & Theory Solver\\
        \hline
 \acsolver      &\smodels~\cite{sim02}& constraint logic programming systems\\
\clingcon & {\clasp}~\cite{geb07}&{\gecode}~\cite{gecode}\\
{\sc idp} & {\sc minisat(id)}~\cite{minisatid-alp}&{\gecode}~\cite{gecode}\\ 
{\sc inca} & {\clasp}~\cite{geb07}& its own CP solver\\
\end{tabular}
\caption{Solvers used by state-of-the-art CASP systems}
\label{tab:solvers}
\end{table}

A few remarks are due. Unlike its peers, \acsolver does not implement learning as its base solver {\smodels} does not support this technique. 
The fact that system {\sc inca} implements its own CP solver, or, in other words, a set of its in house CP-based propagators allows this system to take advantage of some sophisticated techniques stemming from CP. In particular, it implements so called ``lazy nogood generation''. This technique allows one to transfer some of the information stemming from CP-based propagations into a propositional logic program extending the original input to a base solver. 
We also note that the latest version of \clingcon, as well, bypasses the use of \gecode by implementing its own CP-based propagators.
All of the above systems are focused on finite domain integer linear constraints. Some of them allow for global constraints. 

System {\sc dingo}  translates CA programs into  SMT modulo difference logic formulas and applies the SMT solver {\sc z3}~\cite{z3} to find their models. 
Rather than arbitrary integer linear constraints, the system only handles those that fall into the class of difference logic. On the other hand, the system does not pose the restriction of finite domain.
The \ezsmt~\cite{sus16a} solver and the \aspmt~\cite{Bartholomew2014} solvers utilize SMT solvers to process CA programs. Both of these systems may only deal with tight programs.
They allow for arbitrary integer linear constraints.
None of the SMT-based CASP solvers  allow for global constraints in their programs due to the underlying solving technology.

Last but not least, the solver~{\sc mingo} translates CA programs into mixed integer programming expressions and then 
utilizes IBM ILOG {\sc cplex}\footnote{\url{http://www.ibm.com/software/commerce/optimization/cplex-optimizer/}} system to find solutions.  

Susman and Lierler~(\shortcite{sus16a}) provide an experimental analysis of systems from all of the families mentioned.

\section{Conclusions}\label{sec:concl}
In this paper, we  have addressed in a principled way the integration of
answer set solving techniques and constraint solving techniques in
CASP solvers and, in particular, in the realm of the constraint answer set
solver {\ezcsp}.
To begin, we  defined logic programs
with constraint atoms (CA programs). To
bridge the ASP and the constraint programming aspects of such programs, we introduced the
notions of asp-abstractions and csp-abstractions, which allow for
a simple and yet elegant way of defining the extended answer sets of CA programs. 

Next, we  described the syntax of the CASP language used by 
the constraint answer set solver {\ezcsp}, which we call \ez. 
It is worth noting that this paper contains the first detailed and principled account of the syntax of 
the {\ez} language. 
We  relate programs
written in the {\ezcsp} language and CA programs. 
The tight relation between {\ez} programs and CA programs
makes it evident that  the
{\ez} language is a full-fledged constraint answer set programming
formalism. Recall that the {\ezcsp} system originated as an attempt to provide
a simple, flexible framework for modeling constraint satisfaction problems. This yields an  interesting observation: constraint
answer set programming can be seen as a  declarative modeling framework
utilizing  constraint satisfaction solving technology.

In this paper we also drew a parallel between CASP  and SMT. We  used this connection to introduce three
important kinds of integration of CASP solvers: {\bbox} integration,
{\gbox} integration, and {\cbox} integration. 
      We  introduced a graph-based abstract
      framework suitable for  describing the
      {\ezcsp} solving algorithm. The idea of using graph-based
      representations for backtrack-search procedures was pioneered by
      the SAT community.  
      Compared to the use of pseudocode for describing algorithms, such a framework
      allows for simpler descriptions of search algorithms, and is
      well-suited for capturing the similarities and differences of
      the various configurations of {\ezcsp} stemming from different
      integration schemas. 

Finally, we presented an experimental comparison of the various
integration schemas, using the implementation of {\ezcsp} as a
testbed. For the comparison, we used  three challenging
benchmark problems from the {\sl Third  Answer Set Programming
  Competition --  2011} ~\cite{aspcomp3}. The experimental analysis takes
into account how differences in the encoding of the solutions may
influence overall performance by exploiting the components of the
solver in different ways. 
The case study that we conducted clearly illustrates the
influence that integration methods have on the behavior of hybrid 
systems. 
\tcb{ 
The main attractive feature of the {\bbox} integration schema is the ease of inception of a new system. In realm of CASP,
one  may take existing off-the-shelf ASP and CSP tools and connect
them together by simple intermediate translation functions. This facilitates fast
implementation of a prototypical CASP solver. One can then move towards a {\gbox} or {\cbox} architecture in the hope of increased performance when a prototype system proves to be
promising.} 
Yet, our experiments demonstrate that different integration schemas may be of use and importance for different
domain, and that, when it comes to performance, there is no single-best integration schema. 
Thus, systematic means ought to be found for facilitating
building   hybrid  systems supporting various coupling
mechanisms. 
\tcb{
Just as the choice of a particular heuristic for selecting decision
literals is often configurable in SAT or ASP solvers via command line
parameters, 
the choice of integration schema in hybrid systems should be configurable. }
Experimental results also indicate a strong need for
theory propagation. Standard interfaces in both base and theory systems are
required in order to easily build  hybrid systems to support this feature.

Building clear and flexible APIs allowing for
various types of interactions between the solvers seems a necessary 
step towards making the development of hybrid solvers effective.
This work provides evidence for the need of an effort towards 
this goal. \tcg{Many SAT solvers and SMT solvers already come with
  APIs that aim at facilitating extensions of these
  complex software systems. 
We argue for this practice to be adopted by other automated reasoning communities.}

Finally, our study was performed in the realm of CASP
technology, but it translates to SMT as well, given the discussed links 
between the two technologies. Incidentally, this also
brings to 
the surface the importance of establishing means of effective
communication between the two communities of constraint answer set
programming and SMT solving. 


\ \\
\ \\

\bibliography{shared}

\appendix
\section{{\ez} -- The Language of {\ezcsp}}\label{sec:ez-language}
The \ez language is aimed at a convenient specification of a propositional ez-program  
$\cE=\langle E,\cA,\cC,\gamma,D \rangle$. To achieve this, the
language supports an explicit specification of domains and variables,
the use of non-ground rules, and a compact representation of lists in
constraints. We begin by describing the syntax of the language. Next,
we define a mapping from {\ez} programs to propositional ez-programs.

Let $\Sigma_{\ez}=\tbeg C_{\ez},  V_{\ez}, F_{\ez}, R_{\ez} \tend$ be a signature, where
$C_{\ez},  V_{\ez}, F_{\ez}$, and $R_{\ez}$
denote pairwise disjoint sets of constant symbols, non-constraint
variable  symbols, function symbols, and relation symbols
respectively.  
Set $C_{\ez}$ includes symbols for integers and pre-defined constants
($fd$, $q$, $r$), denoting CSP domains. 
We use common convention in logic programming and denote
non-constraint variable  symbols in $V_{\ez}$ by means of upper case
letters. 
Function and relation symbols are associated 
with a non-negative integer called {\em arity}. 
The arity of function symbols is always greater than $0$.
Set $F_{\ez}$ includes pre-defined symbols that intuitively correspond
to arithmetic operators (e.g., $+$), reified arithmetic connectives
(e.g., $<$, $=$), reified logical connectives (see Table
\ref{tab:logic-conn}), list delimiters ($[$ and $]$) and names of
global constraints (discussed later in this section). Set $R_{\ez}$
contains pre-defined symbols $cspdomain$, $cspvar$, $required$. 

The notions of terms, atoms, literals, and rules are defined over $\Sigma_{\ez}$ similarly to ASP, although
the notion of term is slightly expanded. Specifically, a term over signature $\Sigma_{\ez}=\tbeg C_{\ez},  V_{\ez}, F_{\ez}, R_{\ez} \tend$
is defined as:
\begin{enumerate}
\item
a constant symbol from  $C_{\ez}$.
\item
a variable symbol from $V_{\ez}$.
\item
an expression of the form 
\beq
f(t_1,\ldots,t_k), 
\eeq{eq:term1}
where $f$  is a function symbol in  $F_{\ez}$ of arity $k$
and $\tbeg t_1, \ldots, t_k \tend$ are terms from cases 1--3 (If a function symbol is a pre-defined
arithmetic operator, arithmetic connective, or logical connective,
then common infix notation is used.)  
\item an \emph{extensional list}, i.e., an expression of the form $[
  t_1, t_2, \ldots, t_k ]$ where $t_i$'s are terms from cases 1--3.  
\item an \emph{intensional list}, i.e., an expression of the form $[ g
  / k ]$ where $g \in F_{\ez}$ or (with slight abuse of notation) $g
  \in R_{\ez}$ and $k$ is an integer.  
\item a global constraint, i.e., an expression
of the form $f(\lambda_1,\lambda_2,\ldots,\lambda_k)$, 
where $f \in F_{\ez}$ and each $\lambda_i$ is a list.\footnote{In
  constraint satisfaction, global constraints are applied to lists of
  terms of arbitrary length, while local constraints, such as $x>y$,
  apply to a fixed number of arguments. For simplicity, in the
  definition of the language we disregard special cases of global
  constraints, whose arguments are not lists.} 
\end{enumerate}
\begin{table}[hptb]
\begin{center}
\begin{tabular}{|l|l|}
\hline
Connective & Constraint Domain\\
\hline
$\lor$ & Disjunction\\
$\land$ & Conjunction\\
$\backslash$ & Exclusive disjunction\\
$\leftarrow$ or $\rightarrow$ & Implication \\
$\leftrightarrow$ & Equivalence\\
$!$ & Negation\\
\hline
\end{tabular}
\end{center}
\caption{{\ez} Logical Connectives}\label{tab:logic-conn}
\end{table} 
Pre-defined arithmetic and logical connectives from $F_{\ez}$ are
dedicated to the specification of constraints. The connectives are
reified to enable their use within atoms of the form
$required(\beta)$.  Furthermore, the logical connectives enable the 
specification of so called ``reified constraints'' such as:
\beq
x \geq 12 \lor y < 3,
\eeq{eq:reif1}
which specifies that either constraint $x \geq 12$ or constraint $y <
3$ should be satisfied by a solution to a problem containing reified
constraint~\eqref{eq:reif1}. 

An {\ez} program is a pair $\tbeg \Sigma_{\ez}, \Pi \tend$,
where~$\Pi$ is a set of rules over signature $\Sigma_{\ez}$.  
Every {\ez} program is required to contain exactly one fact, whose
head is $cspdomain(fd)$, $cspdomain(q)$, or $cspdomain(r)$. Following
common practice, we denote a program simply by the set of its rules,
and let the signature be implicitly defined. 

Similarly to ASP, a \emph{non-ground rule} is a rule containing one or
more non-constraint variables.
A non-ground rule is interpreted 
as a shorthand for the set of propositional (ground) rules obtained by replacing
every non-constraint variable in the rule by suitable
terms not containing non-constraint variables.
The process of replacing non-ground rules by their propositional
counterparts is called {\em grounding} and is well understood in
ASP~\cite{geb07b,cal08}. 
For this reason, in the rest of this section we focus on ground {\ez} programs.


We now define a mapping from a (ground) {\ez} program $\Pi$ to a
propositional ez-program $\cE=\langle E,\cA,\cC,\gamma,D \rangle$. We
assume that $\gamma$, a function from $\cC$ to constraints, is defined
along the lines of Section \ref{sec:LPCA} and given. 
Recall that only one fact formed from relation $cspdomain$ is allowed
in a program $\Pi$. The fact's head is mapped to the constraint domain~$D$ by mapping $\mu_D$: 
\[
\mu_D(\Pi)=\left\{
\begin{array}{cl}
\mathcal{FD} \mbox{ (finite domains)} & \mbox{ if } cspdomain(fd). \in \Pi\\
\mathcal{Q} & \mbox{ if } cspdomain(q). \in \Pi\\
\mathcal{R} & \mbox{ if } cspdomain(r). \in \Pi\\
\end{array}
\right.
\] 
Atoms formed from relation $cspvar$ specify the set 
$\cV_{{\cP_\cE}}$ of variables (recall that $\cV_{{\cP_\cE}}$ is the
set of constraint variables that appear in csp-abstractions
corresponding to $\cE$). 
 The corresponding atoms take two forms, $cspvar(v)$ and
 $cspvar(v,l,u)$, where $v$ is a term from $\Sigma_{\ez}$ and $l$, $u$
 belong to $C_{\ez} \cap D$. 
The latter form allows one to provide a range for the
variable. Specifically, set $\cV_{{\cP_\cE}}$ is obtained from facts
containing the above atoms as follows: 
\[
\cV_{{\cP_\cE}}=\{ v \,|\, cspvar(v). \in \Pi \mbox{ or } cspvar(v,l,u). \in \Pi \}.
\]
The constraints that specify the range of the variables are generated by mapping $\mu_V$:
\[
\ba{ll}
\mu_V(\Pi)=&\{ required(v \geq l). \,|\, cspvar(v,l,u). \in\Pi \} \ \ \cup\\
&\{ required(v \leq u). \,|\, cspvar(v,l,u). \in\Pi \}.
\ea
\]
Next, we address the specification of lists. Let us begin by
introducing some needed terminology. If a term is of the
form~\eqref{eq:term1} 
then we refer to $f$ as a {\em functor}
and to $\tbeg t_1, \ldots, t_k \tend$ as its {\em arguments.} 
For an atom of the form $r(t_1, \ldots, t_k)$ we say that $r$  is its
{\em relation} and $\tbeg t_1, \ldots, t_k
\tend$ are  its arguments. The expression $terms(f,k,\tbeg t_1, t_2,
\ldots t_m))$ (with $0 \leq m \leq k$) denotes the set of terms from
$\Sigma_{\ez}$ formed by functor $f$ that have arity $k$ and whose
arguments have prefix $\tbeg t_1, t_2, \ldots, t_m \tend$.  
The expression $atoms(r,k,\tbeg t_1, \ldots, t_k \tend)$ denotes the
set of atoms formed by relation $r$ that have arity $k$ and whose
arguments have prefix $\tbeg t_1, t_2, \ldots, t_m \tend$.  The
expression $facts(\Pi)$ denotes the facts in $\Pi$. 
Finally, given a set $S$, $lexord(S)$ denotes a list $[ e_1, e_2,
  \ldots, e_n ]$ enumerating the elements of $S$ in such a way that
$e_i \leq e_{i+1}$ (where $\leq$ denotes lexicographic
ordering\footnote{The choice of  
a particular order is due to the fact that
global constraints that accept multiple lists typically expect
the elements in the same position throughout the lists to
be in a certain relation.
More sophisticated techniques for the specification of lists are
possible, but in our experience, this method 
gives satisfactory results.
}).
We can now define mappings $\lambda_v$ and $\lambda_r$ from the two
forms of intensional lists to corresponding extensional lists:  
\begin{itemize}
\item
 Given an expression of the form $[ f(t_1, t_2, \ldots, t_m) / k ]$,
 where $f \in F_{\ez}$, $k$ is an integer from $C_{\ez}$, $t_i$'s are
 terms, and $0 \leq m \leq k$, 
its \emph{extensional representation} is the list:
\[
\lambda_v([ f(t_1, t_2, \ldots, t_m) / k ])=lexord(terms(f,k,\tbeg t_1, t_2, \ldots, t_m \tend) \cap \cV_{{\cP_\cE}})
\]
of all variables with functor $f$, arity $k$, and
whose arguments have prefix $\tbeg t_1, t_2, \ldots, t_m \tend$.
For example, given a set  of variables $$X_1=\{ v(1), v(2), v(3),
w(a,1), w(a,2), w(b,1) \},$$ the expression $[w(a)/2]$ denotes the 
list $\lambda_v(w,2,\tbeg a \tend)=[ w(a,1), w(a,2) ]$.
When $m=0$, the expression is abbreviated $[ f / k ]$. For instance,
given set $X_1$ as above, the expression $[v/1]$ denotes $[ v(1),
  v(2), v(3) ]$. \item Consider an expression $[r(t_1, t_2, \ldots,
  t_m) / k ]$, where $r$ is not a pre-defined relation from $R_{\ez}$
and $0 \leq m \leq k$. Let $[ a_1, a_2, \ldots, a_n ]$ denote list
$lexord(facts(\Pi) \cap atoms(r,k,\tbeg t_1, \ldots, t_m \tend))$ and
let $\alpha^k_i$ denote the $k^{th}$ argument of $a_i$. Then, the
extensional representation, $\lambda_r([ r(t_1, t_2, \ldots, t_m) / k
])$, of $[r(t_1, t_2, \ldots, t_m) / k ]$ is: 
\[
\lambda_r([ r(t_1, t_2, \ldots, t_m) / k ])= [\alpha_1^k,\alpha_2^k, \ldots, \alpha_n^k].
\]
 For example, given a relation $r'$ defined by facts
$r'(a,1,3), r'(a,2,1), r'(b,5,7)$, the expression $[r'(a)/3]$
denotes the list $[ 3, 1 ]$ and the expression $[r'(a,2)/3]$ denotes $[ 1 ]$.
Similarly to the previous case, when the list of arguments is empty,
the expression can be abbreviated as $[ r / k ].$ For instance, given a
relation $r''$ for which we are given facts 
$r''(a,3), r''(b,1), r''(c,2)$, the expression $[r''/2]$
denotes $\tbeg 3, 1, 2 \tend$.
\end{itemize}
As a practical example of the use of intensional lists, suppose
that, above, relation $r''$ denotes the amount of  
resources required for a job and suppose that we are given facts
$d(a,1), d(b,1), d(c,1)$, specifying  
that jobs $a$, $b$, $c$ have duration $1$. Additionally, variables $st(a)$, $st(b)$, $st(c)$ represent
the start time of the jobs. A cumulative
constraint\footnote{\ref{sec:glc} gives information on
  cumulative and other global constraints.} for this scenario can be
written as 
\[
required(cumulative([st/1], [d/2], [r''/2],4)),
\]
which is an abbreviation of\footnote{Note that the first argument is of the type $[ f(t_1, t_2, \ldots, t_m) / k ]$ while the other two are of type $[ r(t_1, t_2, \ldots, t_m) / k ]$, hence the different expansions.}
\[
required(cumulative([ st(a), st(b), st(c) ], [ 1, 1, 1 ], [ 3, 2, 1 ],4)).
\]
and means that values should be assigned to variables $st(a), st(b), st(c)$ so that each job, of duration $1$ and requiring amounts of resources $3, 2, 1$ respectively, can be executed on a machine that can provide at most $4$ resources at any given time.

Next, let $\mu_{R}$ be a function that maps an atom of the form
$required(\beta)$ to an atom $required(\beta')$  by: 
\begin{itemize}
\item
Replacing every occurrence of $[f(t_1,\ldots,t_m)/k]$ in $\beta$ by $\lambda_v([f(t_1,\ldots,t_m)/k])$;
\item
Replacing every occurrence of $[r(t_1,\ldots,t_m)/k]$ in $\beta$ by $\lambda_r([f(t_1,\ldots,t_m)/k])$.
\end{itemize}
The mapping is easily extended to rules and to programs as follows:
\[
\mu_{R}(a \hif B\ldotp)=\left\{
\begin{array}{ll}
\mu(a) \hif B\ldotp & \mbox{if $a$ is of the form $required(\beta)$} \\
a \hif B\ldotp & \mbox{otherwise}
\end{array}
\right.
\]
where $B$ denotes the body of a rule.

\[
\mu_R(\Pi)=\bigcup_{r \in \Pi} \mu_R(r)
\]
Finally, let $\mu_\cA(\Pi)$ and $\mu_\cC(\Pi)$ denote mappings from
$\Pi$ to alphabets $\cA$ and $\cC$, which are straightforward given
the above construction. Thus, given an {\ez} program $\Pi$, the
corresponding propositional ez-program is: 
\[
\cE(\Pi)=\langle\  \mu_V(\Pi) \cup \mu_R(\Pi),\  \mu_\cA(\Pi),\  \mu_\cC(\Pi),\  \gamma,\  \mu_D(\Pi)\  \rangle.
\]

\subsection{Global Constraints in Language {\ez}}\label{sec:glc}
The global constraints supported by the {\ez} language include:
\begin{itemize}
\item
$all\_dif\!ferent(V)$, where $V$ is a list of variables. This
  constraint, available only in the $fd$ domain, ensures that all the
  variables in $V$ are assigned unique values. 
Typically\footnote{See for example
  \url{http://sicstus.sics.se/sicstus/docs/3.7.1/html/sicstus_33.html}},
the implementation of the  
corresponding algorithm found in constraint solvers is
incomplete. Global constraint $all\_distinct(V)$, which provides a
complete 
implementation of the algorithm, is also supported.
\item
$assignment(X,Y)$, where $X$ and $Y$ are lists of $n$ variables whose domain is $1..n$. The constraint is satisfied if, for every $i,j$, $X_i = j$ if and only if $Y_j = i$.
\item
$circuit(V)$, where $V$ is a list of $n$ variables whose domain is
  $1..n$. The constraint is satisfied by an assignment $V_1=v_1$,
  $V_2=v_2$, $\ldots$, $V_n=v_n$ if the directed graph with nodes $1
  \ldots n$ and arcs $\tbeg 1,v_1 \tend$, $\tbeg 2, v_2 \tend$,
  $\ldots$, $\tbeg n, v_n \tend$ forms a Hamiltonian cycle. 
\item
$count(M,V,\circ,E)$, where $M$ is an integer or variable, $V$ a list
  of variables, $\circ$ an arithmetic comparison operator, and $E$ an
  integer or variable. This constraint is satisfied if the number,
  $c$, of elements of $V$ that equal $M$ is such that $c \circ E$. 
\item
$cumulative(S,D,R,L)$, where $S$ is a list of variables, $D$ and $R$
  are lists of non-negative integers matching the length of $S$, and
  $L$ is an integer or a variable. This constraint, which is only
  available in the $fd$ domain, is typically used in scheduling
  problems. In that context, $S$ represent the start times of a set of
  jobs, $D$ provides the duration of those jobs, and $R$ the resources
  they require. $L$ is the amount of resources available at any time
  step. Intuitively, the constraint assigns start times to the jobs so
  that, at any time, no more than an amount $L$ of resources is used. 
\item
$disjoint2(X,W,Y,H)$, where $X, Y$ are lists of variables and $W, H$
  are lists of integers defining the coordinates and dimensions of
  rectangles. For example, if $X=[x_1, \ldots]$, $Y=[y_1,\ldots]$,
  $W=[w_1,\ldots]$, $H=[h_1,\ldots]$, one of the rectangles they
  describe has top-left vertex $\tbeg x_1, y_1 \tend$ and bottom-right
  vertex $\tbeg x_1+w_1, y_1+h_1 \tend$. This constraint is only
  available in the $fd$ domain, and assigns values to the variables so
  that the corresponding rectangles do not overlap. 
\item
$element(I,V,E)$, where $I$ is an integer or variable, $V$ a list of
  variables,
 and $E$ an integer or variable. This constraint is satisfied if the $I^{th}$ element of $V$ is $E$.
\item
$minimum(M,V)$ and $maximum(M,V)$, where $M$ is a variable or integer
  and $V$ is a list of variables. These constraints are satisfied if
  minimum or  maximum of $V$ equals $M$. 
\item
$scalar\_product(C,X,\circ,E)$, where $C$ is a list of integers, $X$
  is a list of variables, $\circ$ is an arithmetic comparison
  operator, and $E$ is an integer or variable. The intuitive meaning
  of this constraint is that the scalar product, $p$, of the elements
  of $C$ and $X$ must be such that $p \circ E$.
\item
$serialized(S,D)$, where $S$ is a list of variables and $D$ is a list
  of integers, intuitively denoting start time and duration of
  jobs. The constraint assigns start times to the jobs so that their
  execution does not overlap, and can be viewed as a special case of
  $cumulative$. 
\item
$sum(V,\circ,E)$, where $V$ is a list of variables, $\circ$ an
  arithmetic comparison operator, and $E$ is an integer or a
  variable. 
This constraint assigns value to the variables so that $(\sum_{v \in V}v) \circ E$ is satisfied.
\end{itemize}

\end{document}